\documentclass[11pt]{article}

\usepackage[final]{acl}

\usepackage{times}
\usepackage{latexsym}
\usepackage[T1]{fontenc}
\usepackage[utf8]{inputenc}
\usepackage{microtype}
\usepackage{inconsolata}

\usepackage{amsmath,amssymb,amsthm,mathtools,bm}
\usepackage{thmtools}
\declaretheorem[name=Theorem,numberwithin=section]{theorem}
\declaretheorem[name=Proposition,sibling=theorem]{proposition}
\declaretheorem[name=Corollary,sibling=theorem]{corollary}
\declaretheorem[name=Lemma,sibling=theorem]{lemma}
\declaretheorem[name=Definition,style=definition,sibling=theorem]{definition}

\declaretheorem[name=Remark,style=remark,sibling=theorem]{remark}

\usepackage{graphicx}
\usepackage{subcaption}
\usepackage{booktabs}
\usepackage{multirow}
\usepackage{array}
\usepackage{makecell}
\usepackage{longtable}
\usepackage[ruled,linesnumbered]{algorithm2e}
\usepackage{tikz}
\usetikzlibrary{positioning,arrows.meta,calc,fit,backgrounds,shapes.geometric,decorations.pathreplacing}

\usepackage{cleveref}

\usepackage{xcolor}
\usepackage{enumitem}
\usepackage{listings}
\usepackage{CJKutf8}
\setlist{nosep}

\newcommand{\Lar}{\mathcal{L}_{\mathrm{AR}}}
\newcommand{\Lsfr}{\mathcal{L}_{\mathrm{SFR}}}
\newcommand{\Lmtp}{\mathcal{L}_{\mathrm{MTP}}}
\newcommand{\Ltot}{\mathcal{L}_{\mathrm{total}}}
\newcommand{\E}{\mathbb{E}}
\newcommand{\R}{\mathbb{R}}
\newcommand{\KL}{\mathrm{KL}}
\newcommand{\CE}{\mathrm{CE}}
\newcommand{\softmax}{\mathrm{softmax}}
\newcommand{\simplex}[1]{\Delta^{|#1|-1}}
\newcommand{\code}[1]{\texttt{\small #1}}

\newcommand{\SFR}{\textsc{SFR}}
\newcommand{\CSSB}{\textsc{CS-SB}}
\let\SS\undefined
\newcommand{\SS}{\textsc{SS}}

\title{Semantic Flow Regularization:\\
Teaching LLMs to Generate Diverse Yet Coherent Responses}

\author{Kerui Peng, Feifei Li, Xingyu Fan, Wenhui Que\thanks{Corresponding author.} \\
  WeChat, Tencent Inc., Beijing, China \\
  \texttt{\{keruipeng, niyali, fanxfan, victorque\}@tencent.com} \\}

\begin{document}
\maketitle

\begin{abstract}
When large language models are fine-tuned to generate persona- or tone-conditioned responses, their output diversity is severely limited—a failure we term \textbf{Cross-Style Collapse}. We trace this collapse to the cross-entropy objective, which under shared representations tends to suppress diverse continuations. We propose \textbf{Semantic Flow Regularization (\SFR)}, a lightweight auxiliary objective that supervises the backbone with continuous sentence-encoder embeddings of future segments via conditional flow matching. The stochastic flow source preserves multi-modality by construction; the flow-matching head is discarded at inference, adding \textbf{zero deployment cost}. On a large-scale industrial dialogue dataset (Qwen3-32B, 9 personas), SFR improves output diversity, style fidelity, and response quality over SFT. We further validate on the public LiveCodeBench-v5 (Qwen2.5-Coder-7B-Instruct), where SFR consistently improves pass@k, confirming generality beyond stylized dialogue. A controlled comparison on MBPP reveals Multi-Token Prediction to be a degenerate special case of SFR.\footnote{Code: \url{https://github.com/WeAgentAI/SFR}.}

\end{abstract}

\section{Introduction}\label{sec:intro}

LLM-based conversational agents are increasingly used in applications
requiring both factual correctness and stylistic diversity. Our
production deployments include personalized question answering for AI
avatars on WeChat Official Accounts and personalized generation for
WeChat Moments' AI-assisted writing, where responses must reflect each
user's communication style while remaining grounded in relevant
published content. In these settings, the same query can admit many
valid, persona-dependent responses.

Standard supervised fine-tuning (SFT) optimizes a token-level
cross-entropy (CE) loss. When valid continuations are multi-modal, whether across different style conditions or across multiple valid solutions under the same condition, CE under finite-capacity shared representations tends to encourage shared high-probability continuations, suppressing output diversity even when the underlying target distribution is rich---a link between the likelihood objective and generic, diversity-poor responses established at least since \citet{li2015diversity}.
In the stylized-dialogue setting this manifests as what we term Cross-Style Collapse, echoing the text degeneration~\citep{nucleus} and persona homogenization~\citep{chameleonlimit} observed in prior work: swapping the persona on a fixed query often flips only
a greeting token while the response body stays nearly unchanged.
A related symptom
appears in code generation, where pass@$k$ gains plateau because
the model collapses to a single template despite multiple correct
solutions existing.

Several families of existing work address aspects of this gap, but
none fully resolves it. \emph{Decoding-time} methods such as nucleus
sampling~\citep{nucleus}, contrastive
decoding~\citep{li2023contrastive}, classifier-guided
decoders~\citep{ctrl,fudge,gedi} reweight logits without
changing the backbone's internal representations, so the underlying
mode-averaging persists. \emph{Per-style fine-tuning}~\citep{westar}
trains a separate parameter set for each persona, achieving strong stylistic control but incurring storage
and maintenance costs that grow linearly with the number of
styles. \emph{Training-time discrete extensions} such as
unlikelihood~\citep{unlikelihood}, SimCTG~\citep{simctg}, and
multi-token prediction~\citep{gloeckle2024better} extend the
supervision horizon yet remain in one-hot geometry and do not
explicitly model the multi-modal distribution over styles.

We propose \textbf{Semantic Flow Regularization (\SFR)}, a single
auxiliary objective added during fine-tuning that addresses the
problem from a different angle: teaching the backbone to
\emph{see} the future stylistic configuration of each response.
At every response position $t$ a lightweight flow-matching (FM) head
$v_\phi$ learns to transport a random unit-sphere source $z_0$ to the
sentence-encoder embedding $z_1^{(t)}=E(y_{t+1:t+k})$ of the
upcoming suffix, conditioned on the backbone hidden state $h_t$.
Because the source is stochastic, the learned vector field
preserves the full conditional distribution rather than collapsing
to its mean (\Cref{app:theory}); the backbone is thus encouraged
to encode style-discriminative information at every position. At
inference time $v_\phi$ and $E$ are simply deleted---the deployed
model is bit-identical to a standard AR LM with no additional
latency or memory.

Our contributions:
\begin{itemize}
  \item We introduce \SFR{}, a training-time auxiliary that
    injects continuous, segment-level future supervision into AR
    language models at zero inference cost, specifically targeting
    the mode-averaging problem in multi-style fine-tuning.
  \item We show theoretically that stochastic flow-matching
    supervision preserves conditional multi-modality where
    deterministic regression alternatives provably cannot.
  \item We validate \SFR{} on a large-scale industrial stylized-dialogue dataset, showing significant improvements in both output diversity and style fidelity; since this dataset is proprietary, we further confirm the generality of \SFR{} on the public LiveCodeBench-v5 benchmark with consistent pass@k gains. A controlled comparison on MBPP additionally connects \SFR{} to Multi-Token Prediction as a degenerate special case.
\end{itemize}

\begin{figure*}[t]
  \centering
  \begin{minipage}[t]{0.495\textwidth}
    \centering
    \includegraphics[width=\linewidth]{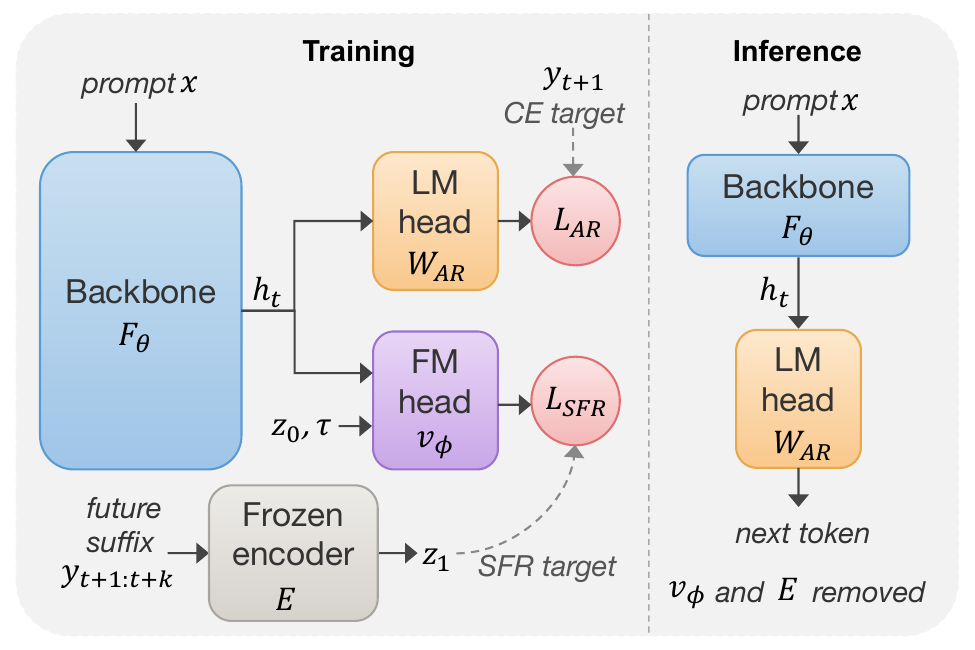}
  \end{minipage}%
  \hfill
  \begin{minipage}[t]{0.495\textwidth}
    \centering
    \includegraphics[width=\linewidth]{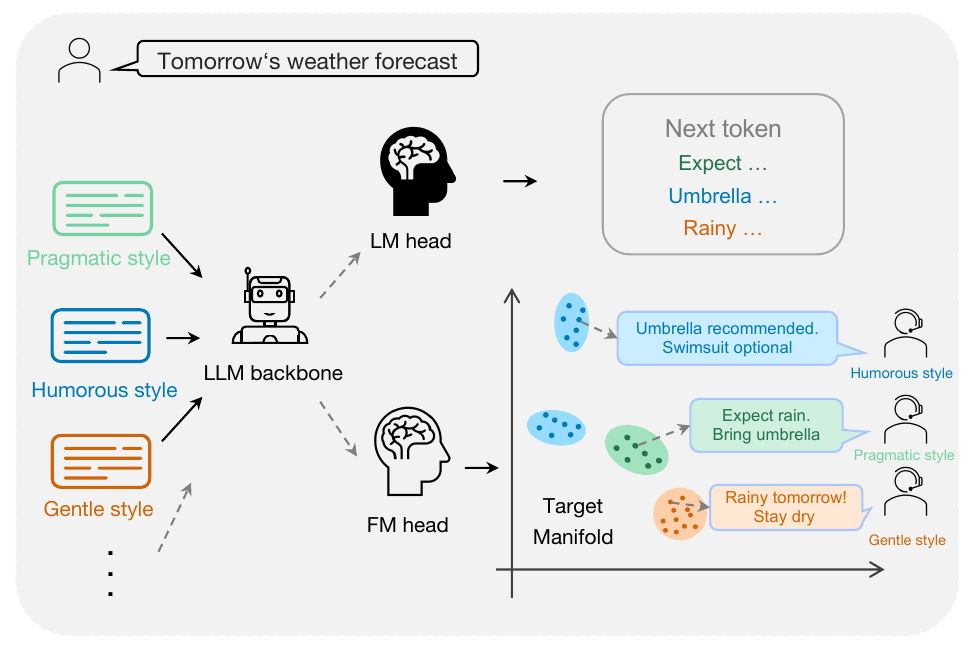}
  \end{minipage}
  \caption{\textbf{Left:} Training and inference architecture of
  \SFR{}. During training, the backbone $F_\theta$ produces hidden
  states $h_t$ consumed by both the LM head $W_{\mathrm{AR}}$
  (standard next-token loss $\Lar$) and the auxiliary FM head
  $v_\phi$ (flow-matching loss $\Lsfr$, supervised by the frozen
  encoder target $z_1=E(y_{t+1:t+k})$). At inference, $v_\phi$ and
  $E$ are discarded; the deployed model is a standard AR LM.
  \textbf{Right:} Illustration of style-discriminative generation.
  The input prompt consists of a style description concatenated with
  the user query. Without \SFR{}, the next-token-only LM head tends toward
  style-collapsed outputs. The FM head regularizes the backbone
  toward distinct regions on the target embedding manifold for each
  style, enabling the model to produce genuinely differentiated
  responses across personas.}
  \label{fig:arch}
\end{figure*}

\section{Related Work}\label{sec:related}

\subsection{Stylized and persona-grounded generation}

Style-controllable dialogue spans both benchmarks and methods.
PersonaChat~\citep{personachat,convai2}, LIGHT~\citep{light},
RoleLLM~\citep{rolellm} and
CharacterGLM~\citep{characterglm} supply persona
profiles and evaluate consistency.
On the method side, per-style fine-tuning~\citep{westar} trains a
dedicated parameter set (or LoRA adapter) per persona,
yielding strong control but linear storage cost in the number of
styles. Decoding-time approaches steer logits without touching the backbone, such as CTRL~\citep{ctrl},
FUDGE~\citep{fudge}, GeDi~\citep{gedi},
PPLM~\citep{pplm}, so the backbone's representations remain mode-averaged and insensitive to the persona prompt.
Closest in spirit are latent-variable dialogue models that also
inject a stochastic signal to escape mode-averaging: VHRED~\citep{vhred}
introduces a hierarchical Gaussian latent over the whole utterance,
and the CVAE dialog model of \citet{cvaedialog} conditions generation
on a sampled discourse-level latent; StyleFusion~\citep{stylefusion}
targets the diversity--relevance trade-off in stylized dialogue most
directly, fusing a style-specific and a content-specific latent space.
All three inject stochasticity into a \emph{discrete-token} decoder
and are trained end-to-end as the primary generative model, so the
stochastic component cannot be removed at inference without
retraining; \SFR{} instead couples the stochastic source to a
\emph{continuous, sentence-level} embedding target via flow matching,
and drops the auxiliary head entirely after training, leaving a
plain AR decoder identical to one trained with SFT. \SFR{} thus
takes a middle path relative to per-style parameters and
decoding-time classifiers: one unified training objective that
makes the backbone itself style-sensitive, requiring neither
per-persona parameters nor decoding-time classifiers.

\subsection{Training-time future supervision}

Several methods augment CE with longer-horizon targets.
Unlikelihood~\citep{unlikelihood}, SimCTG~\citep{simctg} and
MixCE~\citep{mixce} penalise degeneration at horizon $1$.
MTP~\citep{gloeckle2024better,deepseekv3} predicts $K$ future
tokens but remains in discrete one-hot space.
Continuous-target alternatives move into embedding space:
representation alignment via MSE or
cosine~\citep{film} regresses to a fixed embedding, while
JEPA~\citep{jepa} and V-JEPA~\citep{vjepa} predict latent
representations of future inputs in vision. JEPA is architecturally
the closest predecessor---it also uses a predictor network
conditioned on context to forecast future embeddings. However, JEPA
operates in the vision domain without an autoregressive generator,
its predictor is deterministic (collapsing to the conditional
mean; \Cref{prop:modeavg}), and it is not designed as a drop-in
auxiliary for zero-overhead AR deployment. \SFR{} differs on all
three counts: it is integrated
with an AR LLM that directly generates text, it uses a stochastic
FM source to preserve multi-modality, and the FM head is deleted
at deployment with zero inference overhead.

Concurrent work~\citep{ssd2026} observes that token positions
divide into \emph{Locks} (one dominant continuation) and
\emph{Forks} (multiple valid continuations), but does not
explicitly focus on locating which positions fall into which category.
In \Cref{sec:analysis-vfd} we show that the FM head trained by
\SFR{} yields an explicit per-position Lock/Fork diagnostic.

\subsection{Flow and diffusion language models}

Diffusion-LM~\citep{li2022diffusionlm}, SEDD~\citep{sedd},
Plaid~\citep{plaid} and RFFlow~\citep{rfflow} model continuous
multi-modal sequence distributions but replace the AR decoder with
an iterative denoiser, breaking compatibility with standard serving
stacks. The key distinction is that these methods use
flow/diffusion as the \emph{generation} mechanism, whereas \SFR{}
uses it purely as a \emph{training-time regularizer}: the flow head
is discarded before deployment, leaving a plain AR LM with zero
architectural or inference-cost overhead.

\section{Method}\label{sec:method}

\subsection{Overall framework}\label{sec:method-framework}

The backbone $F_\theta$ and the LM head $W_{\mathrm{AR}}$ are
unchanged. We attach a lightweight conditional vector-field head
$v_\phi(z_\tau,\tau;h_t)$ that participates only in the auxiliary
loss and is dropped at inference. Training minimizes
\begin{equation}
  \Ltot \;=\; \Lar \;+\; \lambda\cdot \Lsfr,
  \label{eq:total}
\end{equation}
where $\Lar$ is the mean next-token cross-entropy over all
response positions and $\Lsfr$ is the mean flow-matching loss
defined in \eqref{eq:sfr-loss} below.
\Cref{fig:arch} (left) summarizes the training-time data flow and the
deployment-time graph, which is identical to a vanilla AR LM, while
\Cref{fig:arch} (right) illustrates how the FM head encourages
style-discriminative backbone representations.

In principle, $v_\phi$ can be attached after any intermediate
layer of the backbone: an earlier tap gives the FM head access to
lower-level features but also requires a larger predictor to
bridge the representational gap between that layer and the final
semantic target $z_1$. We attach $v_\phi$ after the last
Transformer layer, so that the conditioning $h_t$ is the same
representation that the LM head (i.e.\ the final un-embedding
projection followed by softmax) operates on. This choice
maximizes the information available to $v_\phi$ while keeping the
predictor small, and ensures that the
FM gradient reshapes precisely the representation that the AR
head consumes. Concretely, $v_\phi$ fuses its three inputs as
follows: the scalar $\tau$ is first mapped, via a sinusoidal
timestep embedding (as in standard diffusion time-conditioning),
to a vector of the same dimension as $z_\tau$; this embedding,
$z_\tau$, and the conditioning hidden state $h_t$ are concatenated
along the feature dimension, passed through an input projection,
and then through several residual blocks with LayerNorm, before a
zero-initialized output projection produces $\hat v$. Hidden width
and depth scale with the backbone (\Cref{app:hparams}). A simple architecture is intentional: the FM
head's role is to supply gradient signal to the backbone rather
than to solve the flow task itself, so limited capacity prevents
it from short-circuiting the auxiliary objective (cf.\
\Cref{sec:method-stab}).

We validate the framework in the SFT setting; extending it to
pretraining, LoRA, or RLHF is left to future work.

\subsection{Auxiliary objective}\label{sec:method-loss}

For each response position $t$, the auxiliary target is the
L2-normalized sentence-encoder embedding of the next $k$ tokens,
\begin{equation}
  z_1^{(t)} \;=\; E\bigl(y_{t+1},\dots,y_{t+k}\bigr)\;\in\;\mathbb{S}^{d_z-1},
  \label{eq:z1}
\end{equation}
where $E$ is a frozen sentence encoder (BGE-Large~\citep{bge} in our
experiments) and $k$ is the suffix horizon. We supervise the FM
head with conditional FM
(CFM;~\citealp{lipman2023flow,liu2023flow}): given a random source
$z_0\sim\mathrm{Unif}(\mathbb S^{d_z-1})$ (a standard Gaussian
sample subsequently renormalized to the unit sphere, matching the
space in which $z_1$ lives), uniform $\tau\sim\mathcal U[0,1]$,
and straight-line interpolation
$z_\tau = (1{-}\tau)z_0 + \tau z_1^{(t)}$, the loss at position
$t$ is
\begin{equation}
  \Lsfr^{(t)} \;=\; \E_{\tau,z_0}\!\bigl\|v_\phi(z_\tau,\tau;h_t) - (z_1^{(t)} - z_0)\bigr\|_2^2.
  \label{eq:sfr-loss}
\end{equation}
The full training loss $\Lsfr$ is the mean of
\eqref{eq:sfr-loss} over all response positions.

Each $h_t$ therefore receives a \emph{continuous}, $d_z$-dimensional
target rather than a one-hot symbol, and a \emph{segment-level}
target spanning the next $k$ tokens rather than only $y_{t+1}$.
The randomness of $z_0$ is also essential rather than cosmetic,
formalized by the following two results (proofs in \Cref{app:theory-fmmm}):

\begin{proposition}[Mode averaging]\label{prop:modeavg}
For a conditional distribution $P(z_1\mid c)$ with finite second
moment, the unique minimizer of $\min_f \E\|f(c)-z_1\|_2^2$ is the
conditional mean $\E[z_1\mid c]$. If $P(z_1\mid c)$ is multi-modal,
this minimizer coincides with no mode and may lie in a low-density
region between them.
\end{proposition}

\begin{proposition}[Multi-modal capacity of CFM]\label{prop:fmmm}
Under the straight-line path $z_\tau=(1{-}\tau)z_0+\tau z_1$ with
stochastic $z_0\sim p_0$, the minimizer of \eqref{eq:sfr-loss} is
the marginal velocity $v^\star(z,\tau,c)=\E[z_1{-}z_0\mid z_\tau{=}z,c]$,
and the ODE $\dot z_\tau=v^\star(z_\tau,\tau,c)$ transports $p_0$ to
the full conditional target distribution $p_1(\cdot\mid c)$ at
$\tau{=}1$, preserving all of its modes.
\end{proposition}

\Cref{prop:modeavg} shows that replacing $z_0$ by a fixed point
collapses \eqref{eq:sfr-loss} into deterministic regression whose
optimum is the conditional mean; \Cref{prop:fmmm} shows that a
stochastic source instead recovers the full conditional
multi-modal structure of $P(z_1\mid h_t)$. The motivation for
choosing flow matching over MSE / cosine alignment or diffusion
denoising is discussed in \Cref{sec:discussion-fm}.

The exact per-position encoding in \eqref{eq:z1} costs $O(L)$ calls
to $E$ per training example. We reduce this to $O(L/s)$ by
\emph{sparse anchor encoding}: $E$ is invoked only every $s$
tokens, and intermediate positions reuse the nearest anchor's
target. The stride ablation in \Cref{app:ablation} shows that
quality is roughly flat from $s{=}1$ to $s{=}16$ while training
time drops by ${\sim}24\%$; in practice $s$ can be chosen jointly
with the suffix length $k$.

\subsection{Inference}\label{sec:method-inference}

At deployment, $v_\phi$ and $E$ are removed and only
$(F_\theta, W_{\mathrm{AR}})$ is loaded. The resulting checkpoint
has the same architecture and inference graph as a vanilla AR LM
and can be served by any
standard tool (vLLM~\citep{vllm}, SGLang~\citep{sglang}, LoRA merging, GPTQ~\citep{gptq}/AWQ~\citep{awq}
quantization) without modification; at deployment \SFR{} adds no
extra parameters, no extra latency, and no extra decoding logic.

The argument for why this is sound is that $v_\phi$ has already
done its work by the end of training: through backpropagation it
has shaped $\theta$ so that $h_t$ encodes future-aware,
condition-discriminative information, and this property travels
with $\theta$ regardless of whether $v_\phi$ is loaded at
inference. \Cref{sec:analysis-tsne} verifies this directly by
showing that the matched-position hidden states
$h_t^{q,s_1},\dots,h_t^{q,s_m}$ become well-separated clusters
under \SFR, even though $v_\phi$ is no longer in the forward
graph.

\subsection{Training stabilisation}\label{sec:method-stab}

Na\"ively adding $\lambda\Lsfr$ to $\Lar$ exhibits two failure
modes. The first is \emph{head short-circuiting}: $v_\phi$ learns
the suffix mapping faster than the backbone reorganizes, $\Lsfr$
collapses to near zero, and the auxiliary gradient stops flowing
back into $\theta$. The second is \emph{early-step shock}: a
randomly initialized $v_\phi$ produces large initial gradients
that destabilize the pretrained backbone before any useful
auxiliary signal has formed. We address both with three
mechanisms that must act jointly---removing any one of them causes
training to diverge (either the backbone loss oscillates
violently or the FM head ceases to influence the backbone at
all).

First, the output projection of $v_\phi$ is zero-initialized, so
that $\hat v\equiv 0$ at step $0$ and $\Lsfr$ starts at the
benign constant $\E\|z_1{-}z_0\|_2^2$, preventing early-step
shock. Second, a two-phase warmup blocks FM gradients from
propagating into the backbone for the first $T_h$ steps (only
$v_\phi$ updates during this period), after which $\lambda$ ramps
linearly from $0$ to its target value over the next $T_\ell$
steps. Third, an exponential moving average with decay $\mu$ is
maintained on $v_\phi$'s parameters, making the predictor a
slow-moving target that the backbone can track without the
predictor racing ahead. Implementation details and schedule values
are given in \Cref{app:implementation,app:hparams};
\Cref{app:pseudocode} provides the complete training algorithm and
\Cref{app:designspace} a taxonomy of the design space.

\section{Experiments}\label{sec:exp}

We evaluate \SFR{} in three settings: (i) a large-scale industrial
stylized-dialogue task that stress-tests output diversity under
persona conditioning (\S\ref{sec:exp-style}), (ii) the public
LiveCodeBench-v5 benchmark that confirms generality on code
generation (\S\ref{sec:exp-lcb}), and (iii) a controlled comparison
with MTP on MBPP (\S\ref{sec:exp-bridge}). Single-factor ablations
on $\lambda$, $k$, $s$, and the three stabilizers are deferred to
\Cref{app:ablation}.

\subsection{Stylized dialogue}\label{sec:exp-style}

To the best of our knowledge, there exists no public dataset that
simultaneously provides articles, user queries, and large-scale
stylized replies from millions of authors. We therefore evaluate
directly on proprietary data from a widely-used, real-world
industrial official-accounts platform. We fine-tune
\textbf{Qwen3-32B}~\citep{qwen3} on ${\sim}170$k multi-persona
dialogue examples covering nine stylistic personas
(\Cref{app:data-style} details the full construction pipeline);
the held-out evaluation set contains ${\sim}150$ queries, and for each
method we generate one response per (query, persona) pair.

We compare against a \textbf{Base} reference (the frozen
Qwen3-32B backbone with persona prompt, without any fine-tuning)
and \textbf{SFT} (standard supervised fine-tuning on the
same multi-persona corpus). Our method, \textbf{\SFR}, is trained
on top of the same SFT recipe with the auxiliary flow-matching
objective of \eqref{eq:sfr-loss}. For the LLM-judge evaluation
we additionally include \textbf{DeepSeek-R1-prompt}, a strong
prompt-only baseline that calls a frozen DeepSeek-R1~\citep{deepseekr1} with the
same persona prompt.

\begin{table}[t]
\centering\small
\setlength{\tabcolsep}{4pt}
\begin{tabular}{lcccc}
\toprule
& SFT & SimCTG & CD & \textbf{\SFR} \\
\midrule
\CSSB$_1$ & $0.798_{\pm.007}$ & $0.813_{\pm.006}$ & $0.812_{\pm.009}$ & $\mathbf{0.783}_{\pm.006}$ \\
\CSSB$_2$ & $0.655_{\pm.011}$ & $0.680_{\pm.011}$ & $0.652_{\pm.015}$ & $\mathbf{0.628}_{\pm.011}$ \\
\CSSB$_3$ & $0.542_{\pm.014}$ & $0.574_{\pm.014}$ & $0.526_{\pm.020}$ & $\mathbf{0.509}_{\pm.013}$ \\
\CSSB$_4$ & $0.452_{\pm.015}$ & $0.487_{\pm.016}$ & $0.432_{\pm.022}$ & $\mathbf{0.418}_{\pm.014}$ \\
\bottomrule
\end{tabular}
\caption{Cross-style Self-BLEU (\CSSB$_n$, $n\in\{1,2,3,4\}$,
lower is better) at $T=0.6$ on Qwen3-32B, \textbf{averaged} over the
$148$-query test set, with $95\%$ paired-bootstrap CIs. SimCTG and
CD are additional diversity-enhancing baselines
(\Cref{app:baselines}).}
\label{tab:style-cssb}
\end{table}

\begin{table}[t]
\centering\small
\begin{tabular}{lccc}
\toprule
& R1-prompt & SFT & \textbf{\SFR{}} \\
\midrule
style-0 Ctx & 4.5 & 4.568 & \textbf{4.622} \\
style-0 Rel & 4.662 & 4.764 & \textbf{4.77} \\
style-0 \SS & 4.654 & 4.588 & \textbf{4.905} \\
style-0 Flu & 4.838 & 4.831 & \textbf{4.905} \\
style-1 Ctx & 4.175 & 4.216 & \textbf{4.257} \\
style-1 Rel & 4.562 & 4.581 & \textbf{4.628} \\
style-1 \SS & 3.117 & 2.858 & \textbf{3.135} \\
style-1 Flu & 4.781 & \textbf{4.838} & 4.791 \\
style-2 Ctx & 4.117 & \textbf{4.169} & 4.128 \\
style-2 Rel & 4.285 & \textbf{4.311} & 4.223 \\
style-2 \SS & 4.766 & 4.73 & \textbf{4.824} \\
style-2 Flu & 4.964 & 4.98 & \textbf{4.993} \\
style-3 Ctx & 4.511 & 4.568 & \textbf{4.696} \\
style-3 Rel & 4.759 & 4.818 & \textbf{4.851} \\
style-3 \SS & 3.891 & 3.932 & \textbf{4.73} \\
style-3 Flu & 4.796 & 4.831 & \textbf{4.98} \\
\midrule
\multicolumn{4}{l}{\footnotesize\emph{The remaining five styles are reported in \Cref{tab:style-scores-full}.}}\\
\midrule
\textbf{Average Ctx} & 4.368 & 4.427 & \textbf{4.463} \\
\textbf{Average Rel} & 4.604 & 4.661 & \textbf{4.666} \\
\textbf{Average \SS} & 4.329 & 4.243 & \textbf{4.401} \\
\textbf{Average Flu} & 4.811 & 4.86 & \textbf{4.881} \\
\bottomrule
\end{tabular}
\caption{Stylized-dialogue LLM-judge scores ($1$--$5$) at $T=0.8$
on Qwen3-32B. \emph{Average}: mean across all nine styles. Full
per-style breakdown in \Cref{app:style-full}.}
\label{tab:style-scores}
\end{table}

For cross-style multi-modality we report \emph{Cross-Style
Self-BLEU} (\CSSB;~\citealp{selfbleu}): for each query, we generate one reply per
style and compute the pairwise Self-BLEU among the nine resulting
replies; the final score averages over all queries in the test set.
Lower \CSSB{} indicates that the model produces genuinely distinct replies
under different persona prompts on the same query. We
additionally report four LLM-judge metrics scored by Qwen3-235B
on a $1$--$5$ Likert scale: \emph{context coherence} (Ctx),
measuring whether the response stays on-topic with respect to
the query; \emph{query relevance} (Rel), measuring whether the
response actually answers what the user asked; \emph{style
strength} (\SS), measuring how strongly the response exhibits
the target persona; and \emph{fluency} (Flu). We aggregate these
scores at the style level, reporting both the per-style
value and the across-style average.

\Cref{tab:style-cssb} reports \CSSB{} averaged across
all queries, with $95\%$ paired-bootstrap confidence intervals
(\Cref{app:significance}). Note that absolute \CSSB{} values are relatively high
because all replies are grounded in the same source article, so
the semantic core of all replies is shared by
design; the metric therefore isolates surface-level stylistic
variation.  \SFR{} consistently outperforms SFT across all
$n$-gram orders, with non-overlapping CIs at every order (e.g.\
$0.418_{\pm.014}$ vs.\ $0.452_{\pm.015}$ at $n{=}4$), indicating that
the flow-matching regularizer further diversifies stylistic
surface form beyond what SFT alone provides.
\textbf{SimCTG} and \textbf{CD} are two additional
diversity-enhancing baselines under the identical setup
(\Cref{app:baselines}); neither matches \SFR{}. SimCTG's \CSSB{}
is close to SFT's, indicating that its contrastive
representation loss does not translate into cross-style
separation. CD's \CSSB{} is competitive with \SFR{}'s (partially
overlapping CIs at $n{\ge}2$), but the LLM-judge breakdown in
\Cref{app:baselines} shows this comes with an abnormally low
style-strength score, indicating a surface-level decoding
artifact rather than genuine style-semantic diversification.
The gain is particularly pronounced on open-ended, opinion-rich
queries where the stylistic space is large (over $33\%$ relative
\CSSB{} reduction over SFT), and smallest on factual or procedural
queries with constrained answer forms; see \Cref{app:cssb-full}
for a per-query breakdown and qualitative analysis.

\begin{figure}[t]
  \centering
  \includegraphics[width=\linewidth]{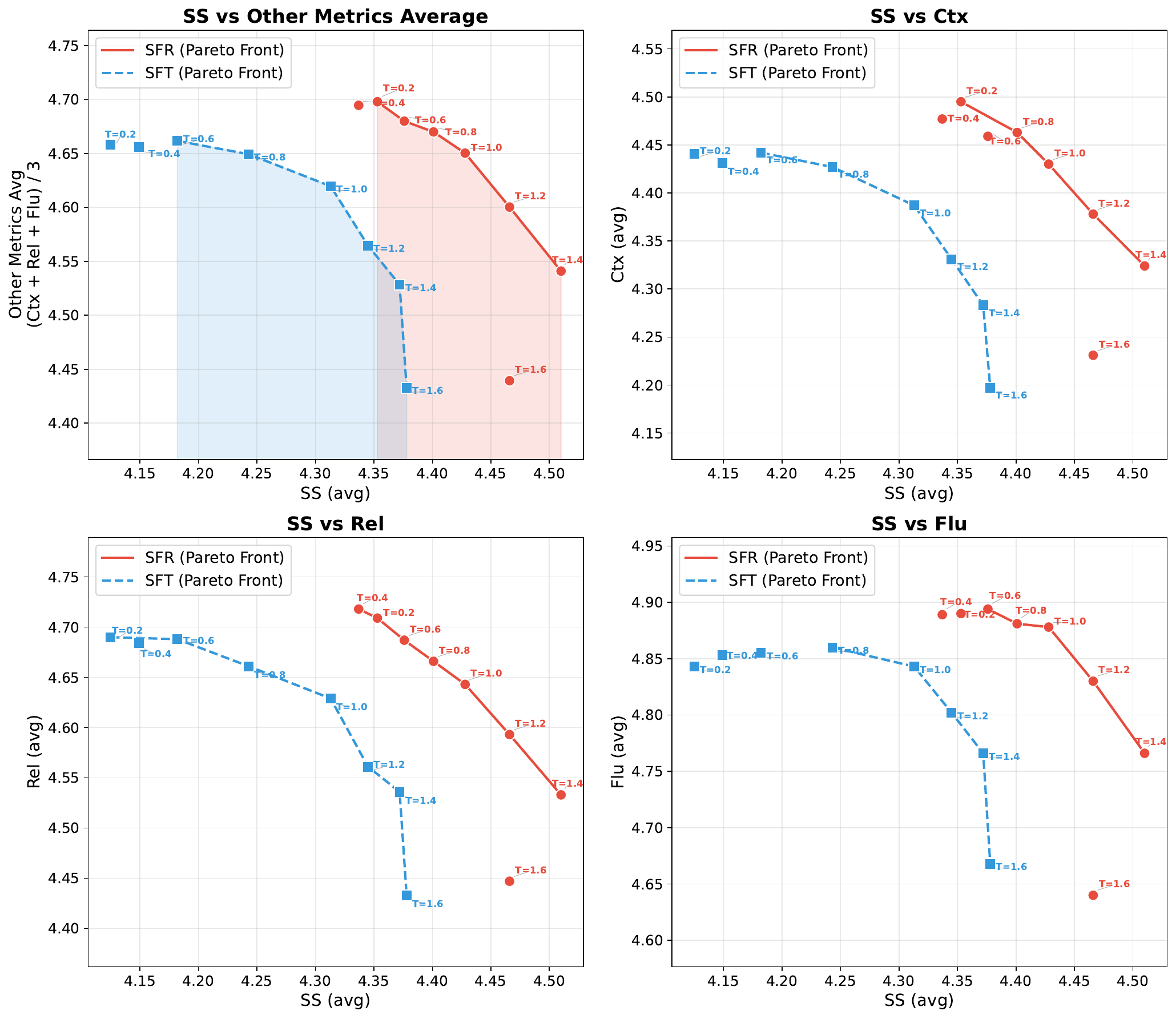}
  \caption{Style strength (\SS) versus other quality metrics
  (Ctx, Rel, Flu) at varying decoding temperatures
  $T\in\{0.2,0.4,0.6,0.8,1.0,1.2,1.4,1.6\}$ on Qwen3-32B.
  \SFR{} yields a better \SS{}--overall-quality trade-off across
  temperatures, with small metric-specific exceptions at high $T$
  where both methods degrade.}
  \label{fig:style-temp}
\end{figure}

\begin{table*}[t]
\centering\small
\setlength{\tabcolsep}{4pt}
\begin{tabular}{llcccccccc}
\toprule
& & \multicolumn{4}{c}{pass@$1$} & \multicolumn{4}{c}{pass@$5$}\\
\cmidrule(lr){3-6}\cmidrule(lr){7-10}
Budget & Method & all & easy & medium & hard & all & easy & medium & hard \\
\midrule
\multirow{2}{*}{$3$B}
 & SFT           & 44.8          & \textbf{92.3} & 51.0          & 13.9          & 58.2          & \textbf{97.2} & 75.1          & 24.4 \\
  & \textbf{\SFR} & \textbf{46.7} & 92.0          & \textbf{55.6} & \textbf{15.1} & \textbf{60.1} & 96.7          & \textbf{78.5} & \textbf{26.6} \\
\midrule
\multirow{2}{*}{$6$B}
 & SFT           & 46.5          & 92.9          & 53.8          & 15.4          & 60.7          & \textbf{97.8} & 77.3          & 28.1 \\
  & \textbf{\SFR} & \textbf{50.2} & \textbf{94.1} & \textbf{59.8} & \textbf{19.0} & \textbf{64.2} & 97.5          & \textbf{80.6} & \textbf{33.9} \\
\midrule
\multirow{2}{*}{$12$B}
 & SFT           & 52.2          & 94.9          & 62.6          & 21.0          & 65.7 & 97.5          & 84.5 & 34.7 \\
  & \textbf{\SFR} & \textbf{54.4} & \textbf{95.9} & \textbf{64.0} & \textbf{24.4} & \textbf{68.7}          & 97.5 & \textbf{86.7}          & \textbf{39.1} \\
\bottomrule
\end{tabular}
\caption{LiveCodeBench-v5 pass@$k$ at $T=0.8$ on
Qwen2.5-Coder-7B-Instruct ($10$ samples per problem), across
three matched training-token budgets ($3$B/$6$B/$12$B).}
\label{tab:lcb-scaling}
\end{table*}

\Cref{tab:style-scores} further shows the four LLM-judge metrics
broken down by representative styles and averaged at $T=0.8$.
\SFR{} improves over both
DeepSeek-R1-prompt and SFT on most styles in aggregate across the
four metrics, with the largest gain on style strength and context
coherence; the simultaneous lift on the average scores indicates
that the gain is genuinely representational rather than achieved
by sacrificing relevance. A random subsample was additionally scored by human annotators;
rankings agreed with the LLM-judge on all four dimensions.

Reading the two tables together with the temperature sweep, the
picture is consistent: \SFR{} simultaneously raises the four
quality scores of \Cref{tab:style-scores} and lowers the
cross-style sameness of \Cref{tab:style-cssb}, and this remains
true across the temperature sweep of \Cref{fig:style-temp}.
This is the regime in which \SFR{} is designed to help: at low
temperature SFT collapses to a near-shared template across
personas, while \SFR{} keeps the persona signal alive throughout
the response.

\paragraph{Judge robustness and public-data validation.}
A cross-family validation judge, \textbf{GLM-5.2}, re-scoring all
dialogue outputs alongside our primary judge (Qwen3-235B), rates
\SFR{} best on every dimension under both judges, with no
same-family bias detected in any \SFR{}-vs-SFT comparison
(\Cref{app:crossjudge}). To rule
out that the diversity gain is an artifact of the proprietary
dataset, we additionally fine-tune Qwen3-4B on the public
PersonaChat~\citep{personachat} dataset
(\Cref{app:personachat} for setup); \Cref{tab:personachat}
shows the same \CSSB{} reduction of \SFR{} over SFT reproduces
across all four $n$-gram orders.

\begin{table}[t]
\centering\small
\begin{tabular}{lcccc}
\toprule
Method & \CSSB$_1$ & \CSSB$_2$ & \CSSB$_3$ & \CSSB$_4$ \\
\midrule
SFT            & 0.653 & 0.507 & 0.391 & 0.299 \\
\textbf{\SFR{}} & \textbf{0.598} & \textbf{0.441} & \textbf{0.322} & \textbf{0.233} \\
\bottomrule
\end{tabular}
\caption{\CSSB{} on public PersonaChat data (Qwen3-4B). Lower is
better ($=$ larger inter-persona diversity).}
\label{tab:personachat}
\end{table}

\subsection{Generalization to open-ended code generation}\label{sec:exp-lcb}

The diversity-preserving mechanism of \SFR{} is not specific to
stylistic variation. It applies whenever the model must maintain
multiple valid generation modes under different conditions. Code
generation, where a single specification admits diverse correct
implementations, provides a natural stress test: here
``conditions'' are problem descriptions rather than style prompts,
and ``modes'' are functionally equivalent but structurally distinct
solutions.

Since the stylized-dialogue dataset is proprietary, we further validate on a public benchmark---LiveCodeBench-v5~\citep{lcb}. We fine-tune
\textbf{Qwen2.5-Coder-7B-Instruct}~\citep{qwencoder} on ${\sim}380$k chain-of-thought
samples cleaned from rStar-Code~\citep{rstarcode}
(\Cref{app:data-lcb}) and stop at three
matched training-token budgets: \textbf{$\mathbf{3}$B}, \textbf{$\mathbf{6}$B}, and
\textbf{$\mathbf{12}$B}, comparing SFT against \SFR{} under matched data,
optimizer, and step counts. Decoding uses temperature $T=0.8$; we
report pass@$1$ and pass@$5$ computed over $10$ independent
samples per problem, broken down by LiveCodeBench's official
difficulty splits (Easy, Medium, Hard) and overall (All).

The result is reported in \Cref{tab:lcb-scaling}. On the
overall (All) split, \SFR{} improves pass@$1$ over SFT at every
budget ($+1.9$, $+3.7$, $+2.2$ at $3$B/$6$B/$12$B), and improves
pass@$5$ consistently ($+1.9$ at $3$B, $+3.5$ at
$6$B, $+3.0$ at $12$B). The gain is most pronounced on the Medium and Hard
splits, where \SFR{} dominates SFT at every budget on pass@$1$
(e.g.\ Medium: $+4.6$, $+6.0$, $+1.4$; Hard: $+1.2$, $+3.6$,
$+3.4$); on the Easy split the two methods are comparable, which
is unsurprising because Easy problems leave little room for
multi-modal coverage to matter. The advantage holds at the
largest budget ($12$B) on both pass@$1$ and pass@$5$ across all
splits (e.g.\ pass@$5$ Medium: $86.7$ vs $84.5$; Hard: $39.1$ vs
$34.7$). The consistent pass@$1$ improvement is in line with the
reading of \SFR{} as enhancing conditional discrimination at the
\emph{single-sample} level: even when coverage is already high,
per-sample quality continues to benefit.

\subsection{Controlled comparison on MBPP}
\label{sec:exp-bridge}

This set of experiments compares \SFR{} against both SFT and
multi-token prediction under strictly matched conditions. On
\textbf{Qwen3-4B-Base}~\citep{qwen3} we train three objectives on ${\sim}500$k
quality-filtered samples from
OpenCodeInstruct~\citep{opencodeinstruct} (\Cref{app:data-mbpp}),
under matched optimizer and step counts: SFT, MTP with $K{=}4$
prediction heads, and \SFR{} encoding the next $k{=}32$ tokens
as its target (\Cref{eq:sfr-loss}), stride $1$. All three checkpoints are evaluated on MBPP
and on the harder MBPP$+$ split~\citep{mbpp,evalplus} at $T=0.6$;
we report pass@$1$, pass@$5$, and pass@$10$, each computed over
$16$ independent samples per problem.

\Cref{tab:bridge} gives the result. \SFR{} achieves the highest
scores on both splits across all three metrics, with a clear
margin over both SFT and MTP on pass@$5$ and pass@$10$
(e.g.\ MBPP$+$ pass@$10$: $79.0$ vs $77.3$ for MTP and $76.7$
for SFT; MBPP-Base pass@$5$: $87.4$ vs $85.8$ for MTP and $85.6$
for SFT). On pass@$1$, \SFR{} also leads on both splits
(MBPP-Base: $78.3$ vs $77.2$/$77.1$; MBPP$+$: $67.7$ vs
$67.2$/$67.0$), though the margin is narrower; the advantage
widens at higher $k$, consistent with the interpretation that
continuous future supervision enriches the backbone's coverage of
alternative solution paths in addition to single-shot accuracy.
Paired-bootstrap $95\%$ CIs on (method $-$ \SFR{}), computed over
problems on the Base split, confirm this pattern is not noise:
neither comparison is significant at pass@$1$ (all three methods
sit near a shared ceiling), but \SFR{}'s advantage becomes
significant against SFT already at pass@$5$, and against both SFT
and MTP at pass@$10$.

\begin{table}[t]
\centering\small
\begin{tabular}{llccc}
\toprule
Method & Split & pass@1 & pass@5 & pass@10 \\
\midrule
SFT       & Base  & 77.2 & 85.6 & 88.1 \\
MTP       & Base  & 77.1 & 85.8 & 88.1 \\
\textbf{\SFR{}} & Base & \textbf{78.3} & \textbf{87.4} & \textbf{89.2} \\
\midrule
SFT       & Plus  & 67 & 74.4 & 76.7 \\
MTP       & Plus  & 67.2 & 75.4 & 77.3 \\
\textbf{\SFR{}} & Plus & \textbf{67.7} & \textbf{76.8} & \textbf{79.0} \\
\bottomrule
\end{tabular}
\caption{MBPP and MBPP$+$ pass@$k$ at $T=0.6$ on Qwen3-4B-Base
($16$ samples per problem). \SFR{} leads on all metrics across
both splits.}
\label{tab:bridge}
\end{table}

Combined with the dialogue and the open-ended code-generation
results, the three experiments tell a consistent story: the
auxiliary flow-matching loss reshapes the backbone's hidden state
so that $h_t$ encodes future-aware, condition-discriminative
information, which manifests as cross-style distinctiveness in
dialogue and as higher pass@$k$ coverage in code generation.
Inference cost is identical to vanilla SFT in all three cases.

\section{Analysis}\label{sec:analysis}

We probe two properties of the trained model: how cleanly the
backbone hidden state separates by conditioning prompt
(\Cref{sec:analysis-tsne}), and how the discarded
flow-matching head, when re-attached \emph{offline}, can be used
to read out a token-level uncertainty signal that the autoregressive
head alone does not expose (\Cref{sec:analysis-vfd}).

\subsection{Hidden states separate by persona}
\label{sec:analysis-tsne}

For each method we collect
\(\{h_t^{q,s}: q\in Q_{\mathrm{eval}},\, s\in S,\, t = t_0\}\),
where $Q_{\mathrm{eval}}$ is the evaluation query set and $S$ the
set of nine personas,
at a fixed response position $t_0$ (the third response token in
the headline result), visualize the resulting vectors with t-SNE~\citep{tsne}
in \Cref{fig:tsne}, and quantify style separation with a
\emph{Style Separability Index},
\(\mathrm{SSI}=\E_{s}\,\mathrm{var}_{q}(h_t^{q,s})/\E_{q}\,\mathrm{var}_{s}(h_t^{q,s})\),
i.e.\ within-style scatter divided by between-style scatter
(lower is better). At $t_0{=}3$ the personas overlap
heavily under SFT and SSI is large ($0.8052$); under \SFR{} the
same styles separate cleanly and SSI drops to $0.5992$,
yielding a $\mathrm{SFT}/\mathrm{SFR}$ ratio of $1.34$. The
persona prompt has therefore been written into the backbone's
representation under \SFR{} and only weakly so under SFT.

\Cref{tab:ssi-by-t0} sweeps the response position $t_0$ and
shows the same trend: \SFR{}'s SSI is at or below SFT's at every
$t_0$ we examined, with the largest gap concentrated in the
$t_0\in[2,5]$ range, where the persona signal is most actively
shaping the early body of the reply; later positions converge as
both methods commit to a particular continuation.

\begin{table}[t]
\centering\small
\begin{tabular}{cccc}
\toprule
$t_0$ & SFT & \SFR{} & ratio \\
\midrule
1   & 0.9705 & \textbf{0.9589} & 1.0121 \\
2   & 0.6793 & \textbf{0.5434} & 1.2501 \\
3   & 0.8052 & \textbf{0.5992} & 1.3439 \\
5   & 0.7915 & \textbf{0.6943} & 1.1401 \\
10  & 1.0467 & \textbf{1.0029} & 1.0437 \\
20  & 1.0431 & \textbf{1.0234} & 1.0193 \\
40  & 1.1030 & \textbf{1.0962} & 1.0061 \\
\bottomrule
\end{tabular}
\caption{Style Separability Index (SSI, lower is better) of the
backbone hidden state at increasing response positions $t_0$.
Ratio is $\mathrm{SFT}/\mathrm{SFR}$; values above $1$ indicate
that \SFR{} produces better-separated persona groups.}
\label{tab:ssi-by-t0}
\end{table}

\begin{figure}[t]
  \centering
  \includegraphics[width=\linewidth]{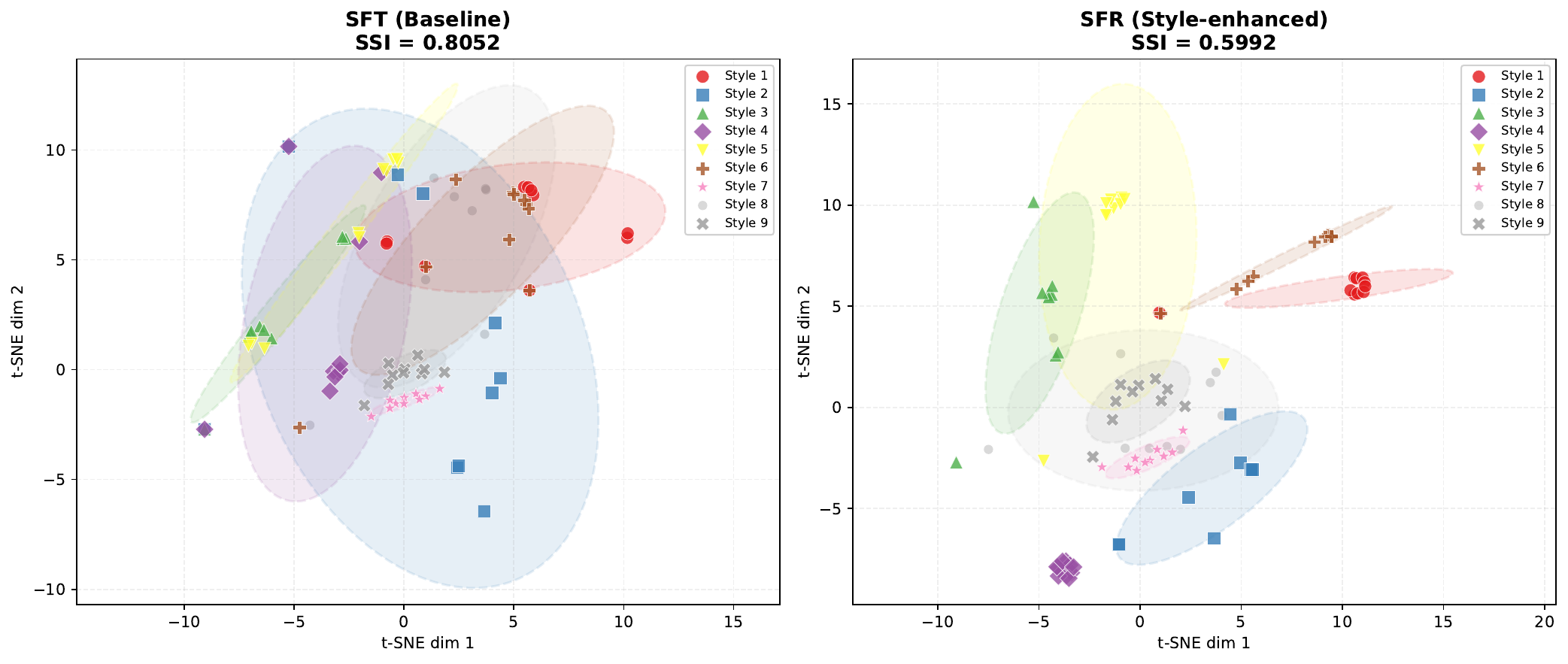}
  \caption{Hidden-state geometry under cross-persona conditioning.
  t-SNE of $h_t^{q,s}$ for ten queries $\times$ nine personas.
  Under SFT the per-persona styles mix across queries; under
  \SFR{} the styles show clearer persona-wise separation across
  the sampled queries.}
  \label{fig:tsne}
\end{figure}
\begin{table*}[t]
\centering\small
\label{tab:vfd-case}

\vspace{0.3em}
{\footnotesize\itshape Prompt:}\;
{\ttfamily\scriptsize Write Python code for quicksort.}

\vspace{0.3em}
{\footnotesize\itshape Completion:}\;
{\ttfamily\scriptsize \textasciigrave\textasciigrave\textasciigrave python \textbackslash ndef quicksort(arr):\textbackslash n\hspace{1em}"""\textbackslash n\hspace{1em}Sorts an array using the quicksort algorithm.\textbackslash n\textbackslash n\hspace{1em}\textcolor{red}{Parameters}:\textbackslash n\hspace{1em}arr (list): The list of elements to be sorted.\textbackslash n\textbackslash n\hspace{1em}\textcolor{red}{Returns}:\textbackslash n\hspace{1em}list: A new list containing all elements from the input list in sorted order.\textbackslash n\hspace{1em}"""\textbackslash n\hspace{1em}if len(arr) <= 1:\textbackslash n\hspace{2em}return arr\textbackslash n\hspace{1em}else:\textbackslash n\hspace{2em}pivot = \textcolor{red}{arr}[len(arr) // 2]\textbackslash n\hspace{2em}\textcolor{red}{left} = [x for x in arr if x < pivot]\textbackslash n\hspace{2em}middle = [x for x in arr if x == pivot]\textbackslash n\hspace{2em}\textcolor{red}{right} = [x for x in arr if x > pivot]\textbackslash n\hspace{2em}\textcolor{red}{return} quicksort(left) + middle + quicksort(right)\textbackslash n\textasciigrave\textasciigrave\textasciigrave}
\vspace{0.3em}

\begin{tabular}{lccccccc}
\toprule
Token  & Parameters & Returns & arr & left & right & return\\
\midrule
$\mathrm{SED}_t$  & 0.669 & 0.8239 & 0.7148 & 0.5951 & 0.4049 & 0.3873 \\
Label  & Fork & Fork & Fork & Lock & Lock & Lock \\
\bottomrule
\end{tabular}
\caption{Lock/Fork diagnosis via $\mathrm{SED}_t$ on a sampled
completion (excerpt). $\mathrm{SED}_t$ values are normalized to
$[0,1]$ by dividing by the per-sequence maximum; the Lock/Fork
label is assigned by thresholding at the per-sequence median.}
\end{table*}

\subsection{Semantic endpoint dispersion as a Lock/Fork diagnostic}
\label{sec:analysis-vfd}

Recent work~\citep{ssd2026} distinguishes \emph{Lock} positions
(one dominant continuation; precision-critical) from \emph{Fork}
positions (multiple valid continuations; exploration-critical),
but does not provide a mechanism to \emph{locate} which positions
are Locks or Forks---the distinction is implicit in the
self-distillation process.
\SFR{}'s
discarded FM head offers an explicit diagnostic: at position $t$ we fix
$h_t$, draw $N{=}64$ independent source samples $z_0^{(i)}\sim\mathrm{Unif}(\mathbb S^{d_z-1})$,
and integrate the learned ODE $\dot z_\tau = v_\phi(z_\tau,\tau;h_t)$
for $50$ Euler steps from $\tau{=}0$ to $\tau{=}1$, yielding
endpoints $\hat z_1^{(i)}$. We then compute the
\emph{semantic endpoint dispersion}
\begin{equation}
\mathrm{SED}_t = \mathrm{tr}\!\bigl[\mathrm{Cov}_{i}(\hat z_1^{(i)})\bigr],
\label{eq:vfd}
\end{equation}
which measures the spread of predicted future semantic endpoints.
Unlike softmax entropy, SED is intrinsic to the
learned flow field and independent of any decoding temperature.
Low SED $\Rightarrow$ Lock; high SED
$\Rightarrow$ Fork at segment level rather than single-token
level. \Cref{tab:vfd-case} shows a concrete example on a
quicksort completion. Tokens whose continuations are nearly
deterministic given the algorithm, such as the idiomatic variable names
\texttt{left}/\texttt{right} and the keyword
\texttt{return}, all receive low SED and are classified as Locks.
In contrast, \texttt{Parameters} and \texttt{Returns} introduce
free-form natural-language descriptions where multiple valid
phrasings exist, and \texttt{arr} appears at the pivot selection
line where the partitioning strategy is genuinely underdetermined;
all three exhibit high SED and are classified as Forks.
We note that the per-sequence median threshold is used here for
illustration; calibrating a task-specific threshold is left for
future work.

\section{Discussion}\label{sec:discussion}

\subsection{Why flow matching}\label{sec:discussion-fm}

Deterministic regression (MSE / cosine to the encoder target)
collapses to the conditional mean $\E[z_1\mid c]$
(\Cref{prop:modeavg}), which mode-averages across conditioning
prompts---precisely the pathology we set out to cure. The
randomness of $z_0$ in \eqref{eq:sfr-loss} is what makes the
optimum equal to the full conditional vector field
(\Cref{prop:fmmm}); swapping FM for MSE or cosine, while keeping
the same auxiliary head and target, erases most of
the gain (\Cref{tab:fm-vs-deterministic}, \Cref{app:ablation-fmvsdet}). FM's constant-magnitude target
and single-step regression are also far simpler to tune than
diffusion DSM as a small auxiliary loss, and the straight-line
formulation makes the following equivalence, which underlies the
MTP inclusion result \Cref{thm:mtp-equiv}, almost trivial to prove:

\begin{lemma}[Velocity-matching $\equiv$ endpoint-matching]\label{lem:vel-eq-end}
Restrict \SFR{} to the discrete, constant-velocity MTP-instance of
\Cref{def:mtp-fm} (\Cref{app:theory-mtp}): the head predicts a
constant velocity $v_{\phi,k}=q_{\phi,k}(\cdot\mid h_t)-\pi$ toward
a fixed reference $\pi$, with target velocity
$u_{\star,k}^{(t)}=e(y_{t+k})-\pi$. Then, for any divergence $D$ on
the tangent space induced by a divergence $D_\Delta$ on the
simplex via translation by $\pi$, matching $v_{\phi,k}$ to
$u_{\star,k}^{(t)}$ under $D$ is identical to matching the endpoint
prediction $q_{\phi,k}(\cdot\mid h_t)$ to the one-hot target
$e(y_{t+k})$ under $D_\Delta$.
\end{lemma}

Taking $D_\Delta$ to be the KL divergence recovers exactly the MTP
cross-entropy loss (\Cref{thm:mtp-equiv}); the full
discussion, including why CFM rather than diffusion, is deferred to \Cref{app:cfm-vs-diffusion,app:theory-mtp}.

\subsection{Scope}\label{sec:discussion-scope}

\SFR{} does not change the inference graph: the deployed
checkpoint has the same architecture as a vanilla AR LM and runs on any
existing serving stack. It also does not verify factual consistency
of responses against the persona specification---that is an
orthogonal fact-checking problem. As \Cref{fig:style-temp}
shows, \SFR{}'s practical edge is largest in the low-temperature
regime that production systems prefer; at $T\ge 1.0$ sampling
randomness dominates and the two methods converge. The formal
relationship between \SFR{} and multi-token prediction is
detailed in \Cref{app:theory}.

\section{Conclusion}\label{sec:conclusion}

We identified Cross-Style Collapse, the failure of fine-tuned LLMs to
produce diverse outputs under varying conditioning, and proposed
Semantic Flow Regularization (\SFR{}), a conditional flow-matching
auxiliary loss whose head is discarded at inference for zero deployment
cost. On industrial dialogue (Qwen3-32B, 9 personas), \SFR{} improves
cross-style diversity while raising quality. Public validation on
LiveCodeBench-v5 and MBPP confirms generality: \SFR{} consistently
lifts pass@$k$ and theoretically subsumes Multi-Token Prediction as a
special case. As a byproduct, the offline FM head yields a
segment-level Lock/Fork uncertainty probe beyond the AR softmax.
Extending \SFR{} to pretraining and RLHF is future work.

\section*{Limitations}\label{sec:limitations}

\textbf{Hyperparameter sensitivity.}\quad
\SFR{} introduces several hyperparameters beyond vanilla SFT:
the loss weight $\lambda$, suffix length $k$, stride $s$, warmup
schedule, and EMA decay, as well as the choice of sentence
encoder $E$ itself. The ablations in \Cref{app:ablation}
demonstrate broad robustness across reasonable ranges, but finding
an optimal configuration still requires non-trivial grid search;
no single default set works universally across all backbones and
tasks.

\textbf{Training dynamics of the FM head.}\quad
The flow-matching loss interacts with the main CE loss in subtle
ways. If $\lambda$ is set too high or the FM head warms up too
quickly, it can ``short-circuit'' the backbone by dominating the
gradient signal before the AR head has converged, degrading
next-token accuracy. Conversely, insufficient warmup leads to
large initial FM errors that propagate noisy gradients into the
backbone. In practice we find it necessary to monitor the FM loss
curve throughout training and adjust the warmup length
accordingly---an extra operational burden absent in standard SFT.

\textbf{Style data coverage and distribution.}\quad
Our experiments use a fixed set of nine hand-curated personas with
roughly balanced data. When the number of styles grows
substantially (e.g.\ hundreds of personas) or the per-style data
distribution is highly imbalanced, minority styles may receive
insufficient gradient signal from the flow-matching objective,
causing the backbone to under-represent them. Whether curriculum
sampling or per-style reweighting can mitigate this is left to
future work. More broadly, our stylized benchmark does not cover
continuous-spectrum styles or cross-language stylization.

\textbf{Scaling to larger backbones.}\quad
Our largest backbone is 32B parameters. Online encoding of suffix
targets adds roughly $10\%$ to training wall-clock time at stride
$4$ (and proportionally less at larger strides), which is the
explicit cost paid for zero inference overhead. Whether the
regularization effect strengthens or weakens at $70$B$+$ scales,
and whether the additional memory footprint of the FM head and
encoder becomes a practical bottleneck at such scales, remain open
questions.

\textbf{Scope of the factuality evidence.}\quad
Our stylized-dialogue task is a \emph{grounded} setting: the
model's input always includes the relevant source article, so a
high context-coherence score is itself a partial indicator that
diversity gains are not coming at the expense of drifting away
from the given context (\Cref{app:rebuttal-extra}). We have not
evaluated \SFR{} in a fully \emph{open-ended} setting where the
model must generate factual content from its own parametric
knowledge without any supporting context (e.g.\ a TruthfulQA-style
benchmark). Our current evidence on the diversity--factuality
trade-off is therefore limited to grounded generation, and
verifying it in open-ended settings is left to future work.

\textbf{Potential risks.}\quad
By explicitly encouraging output diversity, \SFR{} may amplify
the generation of undesirable content if the training data contains
harmful, biased, or toxic material---the method does not
distinguish between ``diverse but safe'' and ``diverse but harmful''
continuations. Standard content-filtering and alignment safeguards
(e.g.\ RLHF, safety classifiers) remain necessary downstream.
Additionally, improved style-mimicking ability could facilitate
misuse in impersonation or disinformation scenarios; deployers
should apply appropriate access controls and monitoring.

% ---- References ----------------------------------------------------
% \bibliographystyle{acl_natbib}
\bibliography{references}

@misc{qwen3,
      title={Qwen3 Technical Report}, 
      author={An Yang and Anfeng Li and Baosong Yang and Beichen Zhang and Binyuan Hui and Bo Zheng and Bowen Yu and Chang Gao and Chengen Huang and Chenxu Lv and Chujie Zheng and Dayiheng Liu and Fan Zhou and Fei Huang and Feng Hu and Hao Ge and Haoran Wei and Huan Lin and Jialong Tang and Jian Yang and Jianhong Tu and Jianwei Zhang and Jianxin Yang and Jiaxi Yang and Jing Zhou and Jingren Zhou and Junyang Lin and Kai Dang and Keqin Bao and Kexin Yang and Le Yu and Lianghao Deng and Mei Li and Mingfeng Xue and Mingze Li and Pei Zhang and Peng Wang and Qin Zhu and Rui Men and Ruize Gao and Shixuan Liu and Shuang Luo and Tianhao Li and Tianyi Tang and Wenbiao Yin and Xingzhang Ren and Xinyu Wang and Xinyu Zhang and Xuancheng Ren and Yang Fan and Yang Su and Yichang Zhang and Yinger Zhang and Yu Wan and Yuqiong Liu and Zekun Wang and Zeyu Cui and Zhenru Zhang and Zhipeng Zhou and Zihan Qiu},
      year={2025},
      eprint={2505.09388},
      archivePrefix={arXiv},
      primaryClass={cs.CL},
      url={https://arxiv.org/abs/2505.09388}, 
}

@inproceedings{li2015diversity,
  author       = {Jiwei Li and
                  Michel Galley and
                  Chris Brockett and
                  Jianfeng Gao and
                  Bill Dolan},
  editor       = {Kevin Knight and
                  Ani Nenkova and
                  Owen Rambow},
  title        = {A Diversity-Promoting Objective Function for Neural Conversation Models},
  booktitle    = {{NAACL} {HLT} 2016, The 2016 Conference of the North American Chapter
                  of the Association for Computational Linguistics: Human Language Technologies,
                  San Diego California, USA, June 12-17, 2016},
  pages        = {110--119},
  publisher    = {The Association for Computational Linguistics},
  year         = {2016},
  url          = {https://doi.org/10.18653/v1/n16-1014},
  doi          = {10.18653/V1/N16-1014},
  bibsource    = {dblp computer science bibliography, https://dblp.org}
}

@inproceedings{vhred,
  author       = {Iulian Vlad Serban and
                  Alessandro Sordoni and
                  Ryan Lowe and
                  Laurent Charlin and
                  Joelle Pineau and
                  Aaron C. Courville and
                  Yoshua Bengio},
  editor       = {Satinder Singh and
                  Shaul Markovitch},
  title        = {A Hierarchical Latent Variable Encoder-Decoder Model for Generating
                  Dialogues},
  booktitle    = {Proceedings of the Thirty-First {AAAI} Conference on Artificial Intelligence,
                  February 4-9, 2017, San Francisco, California, {USA}},
  pages        = {3295--3301},
  publisher    = {{AAAI} Press},
  year         = {2017},
  url          = {https://doi.org/10.1609/aaai.v31i1.10983},
  doi          = {10.1609/AAAI.V31I1.10983},
  bibsource    = {dblp computer science bibliography, https://dblp.org}
}

@inproceedings{cvaedialog,
  author       = {Tiancheng Zhao and
                  Ran Zhao and
                  Maxine Esk{\'{e}}nazi},
  editor       = {Regina Barzilay and
                  Min{-}Yen Kan},
  title        = {Learning Discourse-level Diversity for Neural Dialog Models using
                  Conditional Variational Autoencoders},
  booktitle    = {Proceedings of the 55th Annual Meeting of the Association for Computational
                  Linguistics, {ACL} 2017, Vancouver, Canada, July 30 - August 4, Volume
                  1: Long Papers},
  pages        = {654--664},
  publisher    = {Association for Computational Linguistics},
  year         = {2017},
  url          = {https://doi.org/10.18653/v1/P17-1061},
  doi          = {10.18653/V1/P17-1061},
  bibsource    = {dblp computer science bibliography, https://dblp.org}
}

@inproceedings{stylefusion,
  author       = {Xiang Gao and
                  Yizhe Zhang and
                  Sungjin Lee and
                  Michel Galley and
                  Chris Brockett and
                  Jianfeng Gao and
                  Bill Dolan},
  editor       = {Kentaro Inui and
                  Jing Jiang and
                  Vincent Ng and
                  Xiaojun Wan},
  title        = {Structuring Latent Spaces for Stylized Response Generation},
  booktitle    = {Proceedings of the 2019 Conference on Empirical Methods in Natural
                  Language Processing and the 9th International Joint Conference on
                  Natural Language Processing, {EMNLP-IJCNLP} 2019, Hong Kong, China,
                  November 3-7, 2019},
  pages        = {1814--1823},
  publisher    = {Association for Computational Linguistics},
  year         = {2019},
  url          = {https://doi.org/10.18653/v1/D19-1190},
  doi          = {10.18653/V1/D19-1190},
  bibsource    = {dblp computer science bibliography, https://dblp.org}
}

@misc{chameleonlimit,
  author       = {Yunze Xiao and
                  Vivienne J. Zhang and
                  Chenghao Yang and
                  Ningshan Ma and
                  Weihao Xuan and
                  Jen{-}Tse Huang},
  title        = {The Chameleon's Limit: Investigating Persona Collapse and Homogenization
                  in Large Language Models},
  journal      = {CoRR},
  volume       = {abs/2604.24698},
  year         = {2026},
  url          = {https://doi.org/10.48550/arXiv.2604.24698},
  doi          = {10.48550/ARXIV.2604.24698},
  eprinttype   = {arXiv},
  eprint       = {2604.24698},
  bibsource    = {dblp computer science bibliography, https://dblp.org}
}

@article{opencodeinstruct,
  author       = {Wasi Uddin Ahmad and
                  Aleksander Ficek and
                  Mehrzad Samadi and
                  Jocelyn Huang and
                  Vahid Noroozi and
                  Somshubra Majumdar and
                  Boris Ginsburg},
  title        = {OpenCodeInstruct: {A} Large-scale Instruction Tuning Dataset for Code
                  LLMs},
  journal      = {CoRR},
  volume       = {abs/2504.04030},
  year         = {2025},
  url          = {https://doi.org/10.48550/arXiv.2504.04030},
  doi          = {10.48550/ARXIV.2504.04030},
  eprinttype   = {arXiv},
  eprint       = {2504.04030},
  bibsource    = {dblp computer science bibliography, https://dblp.org}
}

@inproceedings{ostheimer2024,
    title = "Text Style Transfer Evaluation Using Large Language Models",
    author = "Ostheimer, Phil  and
      Nagda, Mayank  and
      Kloft, Marius  and
      Fellenz, Sophie",
    editor = "Calzolari, Nicoletta  and
      Kan, Min-Yen  and
      Hoste, Veronique  and
      Lenci, Alessandro  and
      Sakti, Sakriani  and
      Xue, Nianwen",
    booktitle = "Proceedings of the 2024 Joint International Conference on Computational Linguistics, Language Resources and Evaluation (LREC-COLING 2024)",
    month = may,
    year = "2024",
    address = "Torino, Italia",
    publisher = "ELRA and ICCL",
    url = "https://aclanthology.org/2024.lrec-main.1373/",
    pages = "15802--15822",
}

@article{qwencoder,
  author       = {Binyuan Hui and
                  Jian Yang and
                  Zeyu Cui and
                  Jiaxi Yang and
                  Dayiheng Liu and
                  Lei Zhang and
                  Tianyu Liu and
                  Jiajun Zhang and
                  Bowen Yu and
                  Kai Dang and
                  An Yang and
                  Rui Men and
                  Fei Huang and
                  Xingzhang Ren and
                  Xuancheng Ren and
                  Jingren Zhou and
                  Junyang Lin},
  title        = {Qwen2.5-Coder Technical Report},
  journal      = {CoRR},
  volume       = {abs/2409.12186},
  year         = {2024},
  url          = {https://doi.org/10.48550/arXiv.2409.12186},
  doi          = {10.48550/ARXIV.2409.12186},
  eprinttype   = {arXiv},
  eprint       = {2409.12186},
  bibsource    = {dblp computer science bibliography, https://dblp.org}
}

@inproceedings{gloeckle2024better,
  author       = {Fabian Gloeckle and
                  Badr Youbi Idrissi and
                  Baptiste Rozi{\`{e}}re and
                  David Lopez{-}Paz and
                  Gabriel Synnaeve},
  editor       = {Ruslan Salakhutdinov and
                  Zico Kolter and
                  Katherine A. Heller and
                  Adrian Weller and
                  Nuria Oliver and
                  Jonathan Scarlett and
                  Felix Berkenkamp},
  title        = {Better {\&} Faster Large Language Models via Multi-token Prediction},
  booktitle    = {Forty-first International Conference on Machine Learning, {ICML} 2024,
                  Vienna, Austria, July 21-27, 2024},
  series       = {Proceedings of Machine Learning Research},
  pages        = {15706--15734},
  publisher    = {{PMLR} / OpenReview.net},
  year         = {2024},
  url          = {https://proceedings.mlr.press/v235/gloeckle24a.html},
  bibsource    = {dblp computer science bibliography, https://dblp.org}
}

@misc{deepseekv3,
      title={DeepSeek-V3 Technical Report}, 
      author={DeepSeek-AI and Aixin Liu and Bei Feng and Bing Xue and Bingxuan Wang and Bochao Wu and Chengda Lu and Chenggang Zhao and Chengqi Deng and Chenyu Zhang and Chong Ruan and Damai Dai and Daya Guo and Dejian Yang and Deli Chen and Dongjie Ji and Erhang Li and Fangyun Lin and Fucong Dai and Fuli Luo and Guangbo Hao and Guanting Chen and Guowei Li and H. Zhang and Han Bao and Hanwei Xu and Haocheng Wang and Haowei Zhang and Honghui Ding and Huajian Xin and Huazuo Gao and Hui Li and Hui Qu and J. L. Cai and Jian Liang and Jianzhong Guo and Jiaqi Ni and Jiashi Li and Jiawei Wang and Jin Chen and Jingchang Chen and Jingyang Yuan and Junjie Qiu and Junlong Li and Junxiao Song and Kai Dong and Kai Hu and Kaige Gao and Kang Guan and Kexin Huang and Kuai Yu and Lean Wang and Lecong Zhang and Lei Xu and Leyi Xia and Liang Zhao and Litong Wang and Liyue Zhang and Meng Li and Miaojun Wang and Mingchuan Zhang and Minghua Zhang and Minghui Tang and Mingming Li and Ning Tian and Panpan Huang and Peiyi Wang and Peng Zhang and Qiancheng Wang and Qihao Zhu and Qinyu Chen and Qiushi Du and R. J. Chen and R. L. Jin and Ruiqi Ge and Ruisong Zhang and Ruizhe Pan and Runji Wang and Runxin Xu and Ruoyu Zhang and Ruyi Chen and S. S. Li and Shanghao Lu and Shangyan Zhou and Shanhuang Chen and Shaoqing Wu and Shengfeng Ye and Shengfeng Ye and Shirong Ma and Shiyu Wang and Shuang Zhou and Shuiping Yu and Shunfeng Zhou and Shuting Pan and T. Wang and Tao Yun and Tian Pei and Tianyu Sun and W. L. Xiao and Wangding Zeng and Wanjia Zhao and Wei An and Wen Liu and Wenfeng Liang and Wenjun Gao and Wenqin Yu and Wentao Zhang and X. Q. Li and Xiangyue Jin and Xianzu Wang and Xiao Bi and Xiaodong Liu and Xiaohan Wang and Xiaojin Shen and Xiaokang Chen and Xiaokang Zhang and Xiaosha Chen and Xiaotao Nie and Xiaowen Sun and Xiaoxiang Wang and Xin Cheng and Xin Liu and Xin Xie and Xingchao Liu and Xingkai Yu and Xinnan Song and Xinxia Shan and Xinyi Zhou and Xinyu Yang and Xinyuan Li and Xuecheng Su and Xuheng Lin and Y. K. Li and Y. Q. Wang and Y. X. Wei and Y. X. Zhu and Yang Zhang and Yanhong Xu and Yanhong Xu and Yanping Huang and Yao Li and Yao Zhao and Yaofeng Sun and Yaohui Li and Yaohui Wang and Yi Yu and Yi Zheng and Yichao Zhang and Yifan Shi and Yiliang Xiong and Ying He and Ying Tang and Yishi Piao and Yisong Wang and Yixuan Tan and Yiyang Ma and Yiyuan Liu and Yongqiang Guo and Yu Wu and Yuan Ou and Yuchen Zhu and Yuduan Wang and Yue Gong and Yuheng Zou and Yujia He and Yukun Zha and Yunfan Xiong and Yunxian Ma and Yuting Yan and Yuxiang Luo and Yuxiang You and Yuxuan Liu and Yuyang Zhou and Z. F. Wu and Z. Z. Ren and Zehui Ren and Zhangli Sha and Zhe Fu and Zhean Xu and Zhen Huang and Zhen Zhang and Zhenda Xie and Zhengyan Zhang and Zhewen Hao and Zhibin Gou and Zhicheng Ma and Zhigang Yan and Zhihong Shao and Zhipeng Xu and Zhiyu Wu and Zhongyu Zhang and Zhuoshu Li and Zihui Gu and Zijia Zhu and Zijun Liu and Zilin Li and Ziwei Xie and Ziyang Song and Ziyi Gao and Zizheng Pan},
      year={2025},
      eprint={2412.19437},
      archivePrefix={arXiv},
      primaryClass={cs.CL},
      url={https://arxiv.org/abs/2412.19437}, 
}

@inproceedings{li2022diffusionlm,
  author       = {Xiang Lisa Li and
                  John Thickstun and
                  Ishaan Gulrajani and
                  Percy Liang and
                  Tatsunori B. Hashimoto},
  editor       = {Sanmi Koyejo and
                  S. Mohamed and
                  A. Agarwal and
                  Danielle Belgrave and
                  K. Cho and
                  A. Oh},
  title        = {Diffusion-LM Improves Controllable Text Generation},
  booktitle    = {Advances in Neural Information Processing Systems 35: Annual Conference
                  on Neural Information Processing Systems 2022, NeurIPS 2022, New Orleans,
                  LA, USA, November 28 - December 9, 2022},
  year         = {2022},
  url          = {http://papers.nips.cc/paper\_files/paper/2022/hash/1be5bc25d50895ee656b8c2d9eb89d6a-Abstract-Conference.html},
  bibsource    = {dblp computer science bibliography, https://dblp.org}
}

@inproceedings{sedd,
  author       = {Aaron Lou and
                  Chenlin Meng and
                  Stefano Ermon},
  editor       = {Ruslan Salakhutdinov and
                  Zico Kolter and
                  Katherine A. Heller and
                  Adrian Weller and
                  Nuria Oliver and
                  Jonathan Scarlett and
                  Felix Berkenkamp},
  title        = {Discrete Diffusion Modeling by Estimating the Ratios of the Data Distribution},
  booktitle    = {Forty-first International Conference on Machine Learning, {ICML} 2024,
                  Vienna, Austria, July 21-27, 2024},
  series       = {Proceedings of Machine Learning Research},
  pages        = {32819--32848},
  publisher    = {{PMLR} / OpenReview.net},
  year         = {2024},
  url          = {https://proceedings.mlr.press/v235/lou24a.html},
  bibsource    = {dblp computer science bibliography, https://dblp.org}
}

@inproceedings{plaid,
  author       = {Ishaan Gulrajani and
                  Tatsunori B. Hashimoto},
  editor       = {Alice Oh and
                  Tristan Naumann and
                  Amir Globerson and
                  Kate Saenko and
                  Moritz Hardt and
                  Sergey Levine},
  title        = {Likelihood-Based Diffusion Language Models},
  booktitle    = {Advances in Neural Information Processing Systems 36: Annual Conference
                  on Neural Information Processing Systems 2023, NeurIPS 2023, New Orleans,
                  LA, USA, December 10 - 16, 2023},
  year         = {2023},
  url          = {http://papers.nips.cc/paper\_files/paper/2023/hash/35b5c175e139bff5f22a5361270fce87-Abstract-Conference.html},
  bibsource    = {dblp computer science bibliography, https://dblp.org}
}

@article{rfflow,
  author       = {Qiang Liu},
  title        = {Rectified Flow: {A} Marginal Preserving Approach to Optimal Transport},
  journal      = {CoRR},
  volume       = {abs/2209.14577},
  year         = {2022},
  url          = {https://doi.org/10.48550/arXiv.2209.14577},
  doi          = {10.48550/ARXIV.2209.14577},
  eprinttype   = {arXiv},
  eprint       = {2209.14577},
  bibsource    = {dblp computer science bibliography, https://dblp.org}
}

@inproceedings{lipman2023flow,
  author       = {Yaron Lipman and
                  Ricky T. Q. Chen and
                  Heli Ben{-}Hamu and
                  Maximilian Nickel and
                  Matthew Le},
  title        = {Flow Matching for Generative Modeling},
  booktitle    = {The Eleventh International Conference on Learning Representations,
                  {ICLR} 2023, Kigali, Rwanda, May 1-5, 2023},
  publisher    = {OpenReview.net},
  year         = {2023},
  url          = {https://openreview.net/forum?id=PqvMRDCJT9t},
  bibsource    = {dblp computer science bibliography, https://dblp.org}
}

@inproceedings{liu2023flow,
  author       = {Xingchao Liu and
                  Chengyue Gong and
                  Qiang Liu},
  title        = {Flow Straight and Fast: Learning to Generate and Transfer Data with
                  Rectified Flow},
  booktitle    = {The Eleventh International Conference on Learning Representations,
                  {ICLR} 2023, Kigali, Rwanda, May 1-5, 2023},
  publisher    = {OpenReview.net},
  year         = {2023},
  url          = {https://openreview.net/forum?id=XVjTT1nw5z},
  bibsource    = {dblp computer science bibliography, https://dblp.org}
}

@article{film,
  author       = {{LCM team} and
                  Lo{\"{\i}}c Barrault and
                  Paul{-}Ambroise Duquenne and
                  Maha Elbayad and
                  Artyom Kozhevnikov and
                  Belen Alastruey and
                  Pierre Andrews and
                  Mariano Coria and
                  Guillaume Couairon and
                  Marta R. Costa{-}juss{\`{a}} and
                  David Dale and
                  Hady Elsahar and
                  Kevin Heffernan and
                  Jo{\~{a}}o Maria Janeiro and
                  Tuan Tran and
                  Christophe Ropers and
                  Eduardo S{\'{a}}nchez and
                  Robin San Roman and
                  Alexandre Mourachko and
                  Safiyyah Saleem and
                  Holger Schwenk},
  title        = {Large Concept Models: Language Modeling in a Sentence Representation
                  Space},
  journal      = {CoRR},
  volume       = {abs/2412.08821},
  year         = {2024},
  url          = {https://doi.org/10.48550/arXiv.2412.08821},
  doi          = {10.48550/ARXIV.2412.08821},
  eprinttype   = {arXiv},
  eprint       = {2412.08821},
  bibsource    = {dblp computer science bibliography, https://dblp.org}
}

@article{vllm,
  author       = {Woosuk Kwon and
                  Zhuohan Li and
                  Siyuan Zhuang and
                  Ying Sheng and
                  Lianmin Zheng and
                  Cody Hao Yu and
                  Joseph Gonzalez and
                  Hao Zhang and
                  Ion Stoica},
  editor       = {Jason Flinn and
                  Margo I. Seltzer and
                  Peter Druschel and
                  Antoine Kaufmann and
                  Jonathan Mace},
  title        = {Efficient Memory Management for Large Language Model Serving with
                  {PagedAttention}},
  booktitle    = {Proceedings of the 29th Symposium on Operating Systems Principles,
                  {SOSP} 2023, Koblenz, Germany, October 23-26, 2023},
  pages        = {611--626},
  publisher    = {{ACM}},
  year         = {2023},
  url          = {https://doi.org/10.1145/3600006.3613165},
  doi          = {10.1145/3600006.3613165},
  bibsource    = {dblp computer science bibliography, https://dblp.org}
}

@inproceedings{lcb,
  author       = {Naman Jain and
                  King Han and
                  Alex Gu and
                  Wen{-}Ding Li and
                  Fanjia Yan and
                  Tianjun Zhang and
                  Sida Wang and
                  Armando Solar{-}Lezama and
                  Koushik Sen and
                  Ion Stoica},
  title        = {{LiveCodeBench}: Holistic and Contamination Free Evaluation of Large
                  Language Models for Code},
  booktitle    = {The Thirteenth International Conference on Learning Representations,
                  {ICLR} 2025, Singapore, April 24-28, 2025},
  publisher    = {OpenReview.net},
  year         = {2025},
  url          = {https://openreview.net/forum?id=chfJJYC3iL},
  bibsource    = {dblp computer science bibliography, https://dblp.org}
}

@article{mbpp,
  author       = {Jacob Austin and
                  Augustus Odena and
                  Maxwell I. Nye and
                  Maarten Bosma and
                  Henryk Michalewski and
                  David Dohan and
                  Ellen Jiang and
                  Carrie J. Cai and
                  Michael Terry and
                  Quoc V. Le and
                  Charles Sutton},
  title        = {Program Synthesis with Large Language Models},
  journal      = {CoRR},
  volume       = {abs/2108.07732},
  year         = {2021},
  url          = {https://arxiv.org/abs/2108.07732},
  eprinttype   = {arXiv},
  eprint       = {2108.07732},
  bibsource    = {dblp computer science bibliography, https://dblp.org}
}

@article{bge,
  author       = {Shitao Xiao and
                  Zheng Liu and
                  Peitian Zhang and
                  Niklas Muennighoff and
                  Defu Lian and
                  Jian{-}Yun Nie},
  editor       = {Grace Hui Yang and
                  Hongning Wang and
                  Sam Han and
                  Claudia Hauff and
                  Guido Zuccon and
                  Yi Zhang},
  title        = {C-Pack: Packed Resources For General Chinese Embeddings},
  booktitle    = {Proceedings of the 47th International {ACM} {SIGIR} Conference on
                  Research and Development in Information Retrieval, {SIGIR} 2024, Washington
                  DC, USA, July 14-18, 2024},
  pages        = {641--649},
  publisher    = {{ACM}},
  year         = {2024},
  url          = {https://doi.org/10.1145/3626772.3657878},
  doi          = {10.1145/3626772.3657878},
  bibsource    = {dblp computer science bibliography, https://dblp.org}
}

@article{gneiting2007strictly,
  title={Strictly Proper Scoring Rules, Prediction, and Estimation},
  author={Gneiting, Tilmann and Raftery, Adrian E.},
  journal={Journal of the American Statistical Association},
  volume={102},
  number={477},
  pages={359--378},
  year={2007}
}

@inproceedings{convai2,
  author    = {Dinan, Emily and Logacheva, Varvara and Malykh, Valentin and Miller, Alexander and Shuster, Kurt and Urbanek, Jack and Kiela, Douwe and Szlam, Arthur and Serban, Iulian and Lowe, Ryan and Prabhumoye, Shrimai and Black, Alan W. and Rudnicky, Alexander and Williams, Jason and Pineau, Joelle and Burtsev, Mikhail and Weston, Jason},
  title     = {The Second Conversational Intelligence Challenge (ConvAI2)},
  booktitle = {The NeurIPS '18 Competition},
  editor    = {Escalera, Sergio and Herbrich, Ralf},
  series    = {The Springer Series on Challenges in Machine Learning},
  pages     = {187--208},
  publisher = {Springer},
  address   = {Cham},
  year      = {2020},
  doi       = {10.1007/978-3-030-29135-8_7},
  url       = {https://doi.org/10.1007/978-3-030-29135-8_7}
}

@misc{personachat,
  author       = {Saizheng Zhang and
                  Emily Dinan and
                  Jack Urbanek and
                  Arthur Szlam and
                  Douwe Kiela and
                  Jason Weston},
  editor       = {Iryna Gurevych and
                  Yusuke Miyao},
  title        = {Personalizing Dialogue Agents: {I} have a dog, do you have pets too?},
  booktitle    = {Proceedings of the 56th Annual Meeting of the Association for Computational
                  Linguistics, {ACL} 2018, Melbourne, Australia, July 15-20, 2018, Volume
                  1: Long Papers},
  pages        = {2204--2213},
  publisher    = {Association for Computational Linguistics},
  year         = {2018},
  url          = {https://aclanthology.org/P18-1205/},
  doi          = {10.18653/V1/P18-1205},
  bibsource    = {dblp computer science bibliography, https://dblp.org}
}

@inproceedings{light,
  author       = {Jack Urbanek and
                  Angela Fan and
                  Siddharth Karamcheti and
                  Saachi Jain and
                  Samuel Humeau and
                  Emily Dinan and
                  Tim Rockt{\"{a}}schel and
                  Douwe Kiela and
                  Arthur Szlam and
                  Jason Weston},
  editor       = {Kentaro Inui and
                  Jing Jiang and
                  Vincent Ng and
                  Xiaojun Wan},
  title        = {Learning to Speak and Act in a Fantasy Text Adventure Game},
  booktitle    = {Proceedings of the 2019 Conference on Empirical Methods in Natural
                  Language Processing and the 9th International Joint Conference on
                  Natural Language Processing, {EMNLP-IJCNLP} 2019, Hong Kong, China,
                  November 3-7, 2019},
  pages        = {673--683},
  publisher    = {Association for Computational Linguistics},
  year         = {2019},
  url          = {https://doi.org/10.18653/v1/D19-1062},
  doi          = {10.18653/V1/D19-1062},
  bibsource    = {dblp computer science bibliography, https://dblp.org}
}

@inproceedings{rolellm,
  author       = {Noah Wang and
                  Zhongyuan Peng and
                  Haoran Que and
                  Jiaheng Liu and
                  Wangchunshu Zhou and
                  Yuhan Wu and
                  Hongcheng Guo and
                  Ruitong Gan and
                  Zehao Ni and
                  Jian Yang and
                  Man Zhang and
                  Zhaoxiang Zhang and
                  Wanli Ouyang and
                  Ke Xu and
                  Wenhao Huang and
                  Jie Fu and
                  Junran Peng},
  editor       = {Lun{-}Wei Ku and
                  Andre Martins and
                  Vivek Srikumar},
  title        = {{RoleLLM}: Benchmarking, Eliciting, and Enhancing Role-Playing Abilities
                  of Large Language Models},
  booktitle    = {Findings of the Association for Computational Linguistics, {ACL} 2024,
                  Bangkok, Thailand and virtual meeting, August 11-16, 2024},
  series       = {Findings of {ACL}},
  pages        = {14743--14777},
  publisher    = {Association for Computational Linguistics},
  year         = {2024},
  url          = {https://doi.org/10.18653/v1/2024.findings-acl.878},
  doi          = {10.18653/V1/2024.FINDINGS-ACL.878},
  bibsource    = {dblp computer science bibliography, https://dblp.org}
}

@article{characterglm,
  author       = {Jinfeng Zhou and
                  Zhuang Chen and
                  Dazhen Wan and
                  Bosi Wen and
                  Yi Song and
                  Jifan Yu and
                  Yongkang Huang and
                  Libiao Peng and
                  Jiaming Yang and
                  Xiyao Xiao and
                  Sahand Sabour and
                  Xiaohan Zhang and
                  Wenjing Hou and
                  Yijia Zhang and
                  Yuxiao Dong and
                  Jie Tang and
                  Minlie Huang},
  title        = {{CharacterGLM}: Customizing Chinese Conversational {AI} Characters with
                  Large Language Models},
  journal      = {CoRR},
  volume       = {abs/2311.16832},
  year         = {2023},
  url          = {https://doi.org/10.48550/arXiv.2311.16832},
  doi          = {10.48550/ARXIV.2311.16832},
  eprinttype   = {arXiv},
  eprint       = {2311.16832},
  bibsource    = {dblp computer science bibliography, https://dblp.org}
}

@inproceedings{nucleus,
  author       = {Ari Holtzman and
                  Jan Buys and
                  Li Du and
                  Maxwell Forbes and
                  Yejin Choi},
  title        = {The Curious Case of Neural Text Degeneration},
  booktitle    = {8th International Conference on Learning Representations, {ICLR} 2020,
                  Addis Ababa, Ethiopia, April 26-30, 2020},
  publisher    = {OpenReview.net},
  year         = {2020},
  url          = {https://openreview.net/forum?id=rygGQyrFvH},
  bibsource    = {dblp computer science bibliography, https://dblp.org}
}

@inproceedings{li2023contrastive,
  author       = {Xiang Lisa Li and
                  Ari Holtzman and
                  Daniel Fried and
                  Percy Liang and
                  Jason Eisner and
                  Tatsunori Hashimoto and
                  Luke Zettlemoyer and
                  Mike Lewis},
  editor       = {Anna Rogers and
                  Jordan L. Boyd{-}Graber and
                  Naoaki Okazaki},
  title        = {Contrastive Decoding: Open-ended Text Generation as Optimization},
  booktitle    = {Proceedings of the 61st Annual Meeting of the Association for Computational
                  Linguistics (Volume 1: Long Papers), {ACL} 2023, Toronto, Canada,
                  July 9-14, 2023},
  pages        = {12286--12312},
  publisher    = {Association for Computational Linguistics},
  year         = {2023},
  url          = {https://doi.org/10.18653/v1/2023.acl-long.687},
  doi          = {10.18653/V1/2023.ACL-LONG.687},
  bibsource    = {dblp computer science bibliography, https://dblp.org}
}

@inproceedings{simctg,
  author       = {Yixuan Su and
                  Tian Lan and
                  Yan Wang and
                  Dani Yogatama and
                  Lingpeng Kong and
                  Nigel Collier},
  editor       = {Sanmi Koyejo and
                  S. Mohamed and
                  A. Agarwal and
                  Danielle Belgrave and
                  K. Cho and
                  A. Oh},
  title        = {A Contrastive Framework for Neural Text Generation},
  booktitle    = {Advances in Neural Information Processing Systems 35: Annual Conference
                  on Neural Information Processing Systems 2022, NeurIPS 2022, New Orleans,
                  LA, USA, November 28 - December 9, 2022},
  year         = {2022},
  url          = {http://papers.nips.cc/paper\_files/paper/2022/hash/871cae8f599cb8bbfcb0f58fe1af95ad-Abstract-Conference.html},
  bibsource    = {dblp computer science bibliography, https://dblp.org}
}

@inproceedings{mixce,
  author       = {Shiyue Zhang and
                  Shijie Wu and
                  Ozan Irsoy and
                  Steven Lu and
                  Mohit Bansal and
                  Mark Dredze and
                  David S. Rosenberg},
  editor       = {Anna Rogers and
                  Jordan L. Boyd{-}Graber and
                  Naoaki Okazaki},
  title        = {MixCE: Training Autoregressive Language Models by Mixing Forward and
                  Reverse Cross-Entropies},
  booktitle    = {Proceedings of the 61st Annual Meeting of the Association for Computational
                  Linguistics (Volume 1: Long Papers), {ACL} 2023, Toronto, Canada,
                  July 9-14, 2023},
  pages        = {9027--9050},
  publisher    = {Association for Computational Linguistics},
  year         = {2023},
  url          = {https://doi.org/10.18653/v1/2023.acl-long.502},
  doi          = {10.18653/V1/2023.ACL-LONG.502},
  bibsource    = {dblp computer science bibliography, https://dblp.org}
}

@inproceedings{unlikelihood,
  author       = {Sean Welleck and
                  Ilia Kulikov and
                  Stephen Roller and
                  Emily Dinan and
                  Kyunghyun Cho and
                  Jason Weston},
  title        = {Neural Text Generation With Unlikelihood Training},
  booktitle    = {8th International Conference on Learning Representations, {ICLR} 2020,
                  Addis Ababa, Ethiopia, April 26-30, 2020},
  publisher    = {OpenReview.net},
  year         = {2020},
  url          = {https://openreview.net/forum?id=SJeYe0NtvH},
  bibsource    = {dblp computer science bibliography, https://dblp.org}
}

@article{ctrl,
  author       = {Nitish Shirish Keskar and
                  Bryan McCann and
                  Lav R. Varshney and
                  Caiming Xiong and
                  Richard Socher},
  title        = {{CTRL:} {A} Conditional Transformer Language Model for Controllable
                  Generation},
  journal      = {CoRR},
  volume       = {abs/1909.05858},
  year         = {2019},
  url          = {http://arxiv.org/abs/1909.05858},
  eprinttype   = {arXiv},
  eprint       = {1909.05858},
  bibsource    = {dblp computer science bibliography, https://dblp.org}
}

@inproceedings{fudge,
  author       = {Kevin Yang and
                  Dan Klein},
  editor       = {Kristina Toutanova and
                  Anna Rumshisky and
                  Luke Zettlemoyer and
                  Dilek Hakkani{-}T{\"{u}}r and
                  Iz Beltagy and
                  Steven Bethard and
                  Ryan Cotterell and
                  Tanmoy Chakraborty and
                  Yichao Zhou},
  title        = {{FUDGE:} Controlled Text Generation With Future Discriminators},
  booktitle    = {Proceedings of the 2021 Conference of the North American Chapter of
                  the Association for Computational Linguistics: Human Language Technologies,
                  {NAACL-HLT} 2021, Online, June 6-11, 2021},
  pages        = {3511--3535},
  publisher    = {Association for Computational Linguistics},
  year         = {2021},
  url          = {https://doi.org/10.18653/v1/2021.naacl-main.276},
  doi          = {10.18653/V1/2021.NAACL-MAIN.276},
  bibsource    = {dblp computer science bibliography, https://dblp.org}
}

@inproceedings{gedi,
  author       = {Ben Krause and
                  Akhilesh Deepak Gotmare and
                  Bryan McCann and
                  Nitish Shirish Keskar and
                  Shafiq R. Joty and
                  Richard Socher and
                  Nazneen Fatema Rajani},
  editor       = {Marie{-}Francine Moens and
                  Xuanjing Huang and
                  Lucia Specia and
                  Scott Wen{-}tau Yih},
  title        = {GeDi: Generative Discriminator Guided Sequence Generation},
  booktitle    = {Findings of the Association for Computational Linguistics: {EMNLP}
                  2021, Virtual Event / Punta Cana, Dominican Republic, 16-20 November,
                  2021},
  series       = {Findings of {ACL}},
  pages        = {4929--4952},
  publisher    = {Association for Computational Linguistics},
  year         = {2021},
  url          = {https://doi.org/10.18653/v1/2021.findings-emnlp.424},
  doi          = {10.18653/V1/2021.FINDINGS-EMNLP.424},
  bibsource    = {dblp computer science bibliography, https://dblp.org}
}

@inproceedings{jepa,
  author       = {Mahmoud Assran and
                  Quentin Duval and
                  Ishan Misra and
                  Piotr Bojanowski and
                  Pascal Vincent and
                  Michael G. Rabbat and
                  Yann LeCun and
                  Nicolas Ballas},
  title        = {Self-Supervised Learning from Images with a Joint-Embedding Predictive
                  Architecture},
  booktitle    = {{IEEE/CVF} Conference on Computer Vision and Pattern Recognition,
                  {CVPR} 2023, Vancouver, BC, Canada, June 17-24, 2023},
  pages        = {15619--15629},
  publisher    = {{IEEE}},
  year         = {2023},
  url          = {https://doi.org/10.1109/CVPR52729.2023.01499},
  doi          = {10.1109/CVPR52729.2023.01499},
  bibsource    = {dblp computer science bibliography, https://dblp.org}
}

@article{vjepa,
  author       = {Adrien Bardes and
                  Quentin Garrido and
                  Jean Ponce and
                  Xinlei Chen and
                  Michael Rabbat and
                  Yann LeCun and
                  Mido Assran and
                  Nicolas Ballas},
  title        = {Revisiting Feature Prediction for Learning Visual Representations
                  from Video},
  journal      = {Trans. Mach. Learn. Res.},
  volume       = {2024},
  year         = {2024},
  url          = {https://openreview.net/forum?id=QaCCuDfBk2},
  bibsource    = {dblp computer science bibliography, https://dblp.org}
}

@inproceedings{minsnr,
  author       = {Tiankai Hang and
                  Shuyang Gu and
                  Chen Li and
                  Jianmin Bao and
                  Dong Chen and
                  Han Hu and
                  Xin Geng and
                  Baining Guo},
  title        = {Efficient Diffusion Training via Min-SNR Weighting Strategy},
  booktitle    = {{IEEE/CVF} International Conference on Computer Vision, {ICCV} 2023,
                  Paris, France, October 1-6, 2023},
  pages        = {7407--7417},
  publisher    = {{IEEE}},
  year         = {2023},
  url          = {https://doi.org/10.1109/ICCV51070.2023.00684},
  doi          = {10.1109/ICCV51070.2023.00684},
  bibsource    = {dblp computer science bibliography, https://dblp.org}
}

@inproceedings{edm,
  author       = {Tero Karras and
                  Miika Aittala and
                  Timo Aila and
                  Samuli Laine},
  editor       = {Sanmi Koyejo and
                  S. Mohamed and
                  A. Agarwal and
                  Danielle Belgrave and
                  K. Cho and
                  A. Oh},
  title        = {Elucidating the Design Space of Diffusion-Based Generative Models},
  booktitle    = {Advances in Neural Information Processing Systems 35: Annual Conference
                  on Neural Information Processing Systems 2022, NeurIPS 2022, New Orleans,
                  LA, USA, November 28 - December 9, 2022},
  year         = {2022},
  url          = {http://papers.nips.cc/paper\_files/paper/2022/hash/a98846e9d9cc01cfb87eb694d946ce6b-Abstract-Conference.html},
  bibsource    = {dblp computer science bibliography, https://dblp.org}
}

@inproceedings{vpred,
  author       = {Tim Salimans and
                  Jonathan Ho},
  title        = {Progressive Distillation for Fast Sampling of Diffusion Models},
  booktitle    = {The Tenth International Conference on Learning Representations, {ICLR}
                  2022, Virtual Event, April 25-29, 2022},
  publisher    = {OpenReview.net},
  year         = {2022},
  url          = {https://openreview.net/forum?id=TIdIXIpzhoI},
  bibsource    = {dblp computer science bibliography, https://dblp.org}
}

@inproceedings{evalplus,
  author       = {Jiawei Liu and
                  Chunqiu Steven Xia and
                  Yuyao Wang and
                  Lingming Zhang},
  editor       = {Alice Oh and
                  Tristan Naumann and
                  Amir Globerson and
                  Kate Saenko and
                  Moritz Hardt and
                  Sergey Levine},
  title        = {Is Your Code Generated by ChatGPT Really Correct? Rigorous Evaluation
                  of Large Language Models for Code Generation},
  booktitle    = {Advances in Neural Information Processing Systems 36: Annual Conference
                  on Neural Information Processing Systems 2023, NeurIPS 2023, New Orleans,
                  LA, USA, December 10 - 16, 2023},
  year         = {2023},
  url          = {http://papers.nips.cc/paper\_files/paper/2023/hash/43e9d647ccd3e4b7b5baab53f0368686-Abstract-Conference.html},
  bibsource    = {dblp computer science bibliography, https://dblp.org}
}

@inproceedings{pplm,
  author       = {Sumanth Dathathri and
                  Andrea Madotto and
                  Janice Lan and
                  Jane Hung and
                  Eric Frank and
                  Piero Molino and
                  Jason Yosinski and
                  Rosanne Liu},
  title        = {Plug and Play Language Models: {A} Simple Approach to Controlled Text
                  Generation},
  booktitle    = {8th International Conference on Learning Representations, {ICLR} 2020,
                  Addis Ababa, Ethiopia, April 26-30, 2020},
  publisher    = {OpenReview.net},
  year         = {2020},
  url          = {https://openreview.net/forum?id=H1edEyBKDS},
  bibsource    = {dblp computer science bibliography, https://dblp.org}
}

@article{ssd2026,
  author       = {Ruixiang Zhang and
                  Richard He Bai and
                  Huangjie Zheng and
                  Navdeep Jaitly and
                  Ronan Collobert and
                  Yizhe Zhang},
  title        = {Embarrassingly Simple Self-Distillation Improves Code Generation},
  journal      = {CoRR},
  volume       = {abs/2604.01193},
  year         = {2026},
  url          = {https://doi.org/10.48550/arXiv.2604.01193},
  doi          = {10.48550/ARXIV.2604.01193},
  eprinttype   = {arXiv},
  eprint       = {2604.01193},
  bibsource    = {dblp computer science bibliography, https://dblp.org}
}

@article{rstarcode,
  author       = {Yifei Liu and
                  Li Lyna Zhang and
                  Yi Zhu and
                  Bingcheng Dong and
                  Xudong Zhou and
                  Ning Shang and
                  Fan Yang and
                  Mao Yang},
  title        = {rStar-Coder: Scaling Competitive Code Reasoning with a Large-Scale
                  Verified Dataset},
  journal      = {CoRR},
  volume       = {abs/2505.21297},
  year         = {2025},
  url          = {https://doi.org/10.48550/arXiv.2505.21297},
  doi          = {10.48550/ARXIV.2505.21297},
  eprinttype   = {arXiv},
  eprint       = {2505.21297},
  bibsource    = {dblp computer science bibliography, https://dblp.org}
}

@article{westar,
  author       = {Xingyu Fan and
                  Feifei Li and
                  Wenhui Que and
                  Hailong Li},
  title        = {One Agent to Serve All: a Lite-Adaptive Stylized {AI} Assistant for
                  Millions of Multi-Style Official Accounts},
  journal      = {CoRR},
  volume       = {abs/2509.17788},
  year         = {2025},
  url          = {https://doi.org/10.48550/arXiv.2509.17788},
  doi          = {10.48550/ARXIV.2509.17788},
  eprinttype   = {arXiv},
  eprint       = {2509.17788},
  bibsource    = {dblp computer science bibliography, https://dblp.org}
}

@inproceedings{selfbleu,
  author       = {Yaoming Zhu and
                  Sidi Lu and
                  Lei Zheng and
                  Jiaxian Guo and
                  Weinan Zhang and
                  Jun Wang and
                  Yong Yu},
  editor       = {Kevyn Collins{-}Thompson and
                  Qiaozhu Mei and
                  Brian D. Davison and
                  Yiqun Liu and
                  Emine Yilmaz},
  title        = {Texygen: {A} Benchmarking Platform for Text Generation Models},
  booktitle    = {The 41st International {ACM} {SIGIR} Conference on Research {\&}
                  Development in Information Retrieval, {SIGIR} 2018, Ann Arbor, MI,
                  USA, July 08-12, 2018},
  pages        = {1097--1100},
  publisher    = {{ACM}},
  year         = {2018},
  url          = {https://doi.org/10.1145/3209978.3210080},
  doi          = {10.1145/3209978.3210080},
  bibsource    = {dblp computer science bibliography, https://dblp.org}
}

@article{tsne,
  author  = {Laurens van der Maaten and Geoffrey Hinton},
  title   = {Visualizing Data using {t-SNE}},
  journal = {Journal of Machine Learning Research},
  year    = {2008},
  volume  = {9},
  number  = {86},
  pages   = {2579--2605},
  url     = {http://jmlr.org/papers/v9/vandermaaten08a.html}
}

@inproceedings{sglang,
  author       = {Lianmin Zheng and
                  Liangsheng Yin and
                  Zhiqiang Xie and
                  Chuyue Sun and
                  Jeff Huang and
                  Cody Hao Yu and
                  Shiyi Cao and
                  Christos Kozyrakis and
                  Ion Stoica and
                  Joseph E. Gonzalez and
                  Clark W. Barrett and
                  Ying Sheng},
  editor       = {Amir Globersons and
                  Lester Mackey and
                  Danielle Belgrave and
                  Angela Fan and
                  Ulrich Paquet and
                  Jakub M. Tomczak and
                  Cheng Zhang},
  title        = {SGLang: Efficient Execution of Structured Language Model Programs},
  booktitle    = {Advances in Neural Information Processing Systems 38: Annual Conference
                  on Neural Information Processing Systems 2024, NeurIPS 2024, Vancouver,
                  BC, Canada, December 10 - 15, 2024},
  year         = {2024},
  url          = {http://papers.nips.cc/paper\_files/paper/2024/hash/724be4472168f31ba1c9ac630f15dec8-Abstract-Conference.html},
  bibsource    = {dblp computer science bibliography, https://dblp.org}
}

@article{gptq,
  author       = {Elias Frantar and
                  Saleh Ashkboos and
                  Torsten Hoefler and
                  Dan Alistarh},
  title        = {{GPTQ:} Accurate Post-Training Quantization for Generative Pre-trained
                  Transformers},
  journal      = {CoRR},
  volume       = {abs/2210.17323},
  year         = {2022},
  url          = {https://doi.org/10.48550/arXiv.2210.17323},
  doi          = {10.48550/ARXIV.2210.17323},
  eprinttype   = {arXiv},
  eprint       = {2210.17323},
  bibsource    = {dblp computer science bibliography, https://dblp.org}
}

@inproceedings{awq,
  author       = {Ji Lin and
                  Jiaming Tang and
                  Haotian Tang and
                  Shang Yang and
                  Wei{-}Ming Chen and
                  Wei{-}Chen Wang and
                  Guangxuan Xiao and
                  Xingyu Dang and
                  Chuang Gan and
                  Song Han},
  editor       = {Phillip B. Gibbons and
                  Gennady Pekhimenko and
                  Christopher De Sa},
  title        = {{AWQ:} Activation-aware Weight Quantization for On-Device {LLM} Compression
                  and Acceleration},
  booktitle    = {Proceedings of the Seventh Annual Conference on Machine Learning and
                  Systems, MLSys 2024, Santa Clara, CA, USA, May 13-16, 2024},
  publisher    = {mlsys.org},
  year         = {2024},
  url          = {https://proceedings.mlsys.org/paper\_files/paper/2024/hash/42a452cbafa9dd64e9ba4aa95cc1ef21-Abstract-Conference.html},
  bibsource    = {dblp computer science bibliography, https://dblp.org}
}

@inproceedings{zero,
  author       = {Samyam Rajbhandari and
                  Jeff Rasley and
                  Olatunji Ruwase and
                  Yuxiong He},
  editor       = {Christine Cuicchi and
                  Irene Qualters and
                  William T. Kramer},
  title        = {ZeRO: memory optimizations toward training trillion parameter models},
  booktitle    = {Proceedings of the International Conference for High Performance Computing,
                  Networking, Storage and Analysis, {SC} 2020, Virtual Event / Atlanta,
                  Georgia, USA, November 9-19, 2020},
  pages        = {20},
  publisher    = {{IEEE/ACM}},
  year         = {2020},
  url          = {https://doi.org/10.1109/SC41405.2020.00024},
  doi          = {10.1109/SC41405.2020.00024},
  bibsource    = {dblp computer science bibliography, https://dblp.org}
}

@article{deepseekr1,
   title={DeepSeek-R1 incentivizes reasoning in {LLMs} through reinforcement learning},
   volume={645},
   ISSN={1476-4687},
   url={http://dx.doi.org/10.1038/s41586-025-09422-z},
   DOI={10.1038/s41586-025-09422-z},
   number={8081},
   journal={Nature},
   publisher={Springer Science and Business Media LLC},
   author={Guo, Daya and Yang, Dejian and Zhang, Haowei and Song, Junxiao and Wang, Peiyi and Zhu, Qihao and Xu, Runxin and Zhang, Ruoyu and Ma, Shirong and Bi, Xiao and Zhang, Xiaokang and Yu, Xingkai and Wu, Yu and Wu, Z. F. and Gou, Zhibin and Shao, Zhihong and Li, Zhuoshu and Gao, Ziyi and Liu, Aixin and Xue, Bing and Wang, Bingxuan and Wu, Bochao and Feng, Bei and Lu, Chengda and Zhao, Chenggang and Deng, Chengqi and Ruan, Chong and Dai, Damai and Chen, Deli and Ji, Dongjie and Li, Erhang and Lin, Fangyun and Dai, Fucong and Luo, Fuli and Hao, Guangbo and Chen, Guanting and Li, Guowei and Zhang, H. and Xu, Hanwei and Ding, Honghui and Gao, Huazuo and Qu, Hui and Li, Hui and Guo, Jianzhong and Li, Jiashi and Chen, Jingchang and Yuan, Jingyang and Tu, Jinhao and Qiu, Junjie and Li, Junlong and Cai, J. L. and Ni, Jiaqi and Liang, Jian and Chen, Jin and Dong, Kai and Hu, Kai and You, Kaichao and Gao, Kaige and Guan, Kang and Huang, Kexin and Yu, Kuai and Wang, Lean and Zhang, Lecong and Zhao, Liang and Wang, Litong and Zhang, Liyue and Xu, Lei and Xia, Leyi and Zhang, Mingchuan and Zhang, Minghua and Tang, Minghui and Zhou, Mingxu and Li, Meng and Wang, Miaojun and Li, Mingming and Tian, Ning and Huang, Panpan and Zhang, Peng and Wang, Qiancheng and Chen, Qinyu and Du, Qiushi and Ge, Ruiqi and Zhang, Ruisong and Pan, Ruizhe and Wang, Runji and Chen, R. J. and Jin, R. L. and Chen, Ruyi and Lu, Shanghao and Zhou, Shangyan and Chen, Shanhuang and Ye, Shengfeng and Wang, Shiyu and Yu, Shuiping and Zhou, Shunfeng and Pan, Shuting and Li, S. S. and Zhou, Shuang and Wu, Shaoqing and Yun, Tao and Pei, Tian and Sun, Tianyu and Wang, T. and Zeng, Wangding and Liu, Wen and Liang, Wenfeng and Gao, Wenjun and Yu, Wenqin and Zhang, Wentao and Xiao, W. L. and An, Wei and Liu, Xiaodong and Wang, Xiaohan and Chen, Xiaokang and Nie, Xiaotao and Cheng, Xin and Liu, Xin and Xie, Xin and Liu, Xingchao and Yang, Xinyu and Li, Xinyuan and Su, Xuecheng and Lin, Xuheng and Li, X. Q. and Jin, Xiangyue and Shen, Xiaojin and Chen, Xiaosha and Sun, Xiaowen and Wang, Xiaoxiang and Song, Xinnan and Zhou, Xinyi and Wang, Xianzu and Shan, Xinxia and Li, Y. K. and Wang, Y. Q. and Wei, Y. X. and Zhang, Yang and Xu, Yanhong and Li, Yao and Zhao, Yao and Sun, Yaofeng and Wang, Yaohui and Yu, Yi and Zhang, Yichao and Shi, Yifan and Xiong, Yiliang and He, Ying and Piao, Yishi and Wang, Yisong and Tan, Yixuan and Ma, Yiyang and Liu, Yiyuan and Guo, Yongqiang and Ou, Yuan and Wang, Yuduan and Gong, Yue and Zou, Yuheng and He, Yujia and Xiong, Yunfan and Luo, Yuxiang and You, Yuxiang and Liu, Yuxuan and Zhou, Yuyang and Zhu, Y. X. and Huang, Yanping and Li, Yaohui and Zheng, Yi and Zhu, Yuchen and Ma, Yunxian and Tang, Ying and Zha, Yukun and Yan, Yuting and Ren, Z. Z. and Ren, Zehui and Sha, Zhangli and Fu, Zhe and Xu, Zhean and Xie, Zhenda and Zhang, Zhengyan and Hao, Zhewen and Ma, Zhicheng and Yan, Zhigang and Wu, Zhiyu and Gu, Zihui and Zhu, Zijia and Liu, Zijun and Li, Zilin and Xie, Ziwei and Song, Ziyang and Pan, Zizheng and Huang, Zhen and Xu, Zhipeng and Zhang, Zhongyu and Zhang, Zhen},
   year={2025},
   month=Sept, pages={633–638} }

% ---- Appendices ----------------------------------------------------
\appendix
\section{Theoretical Foundations}\label{app:theory}

This appendix collects the formal statements and full proofs that
support the claims made in \Cref{sec:method-loss,sec:discussion-fm,sec:discussion-scope}.
We first set up notation (\Cref{app:theory-prelim}), then prove the
multi-modality argument for choosing CFM over deterministic
alignment (\Cref{app:theory-fmmm}), then prove the strict-equivalence
results that place MTP as a special case of \SFR{}
(\Cref{app:theory-mtp}), and finally summarize the resulting
inclusion chain (\Cref{app:theory-chain}). \Cref{app:cfm-vs-diffusion}
addresses the related question of why CFM rather than diffusion.

\subsection{Notation and the variational view}\label{app:theory-prelim}

Let $V$ be the vocabulary, $|V|$ its size, and
$\Delta := \simplex{V}$ the probability simplex. For a token $v\in V$
write $e(v)\in\Delta$ for its one-hot embedding. We use $h_t\in\R^d$
for a backbone hidden state and $z_1\in\R^{d_z}$ for an auxiliary
target (a continuous embedding in the body of the paper, a point on
$\Delta^K$ in \Cref{def:mtp-fm}).

The token-level CE loss and the FM auxiliary loss used in
the body are
\begin{align}
\Lar &= \sum_t \CE\bigl(\softmax(W_{\mathrm{AR}} h_t),\,y_{t+1}\bigr),\\
\Lsfr &= \E_{\tau,z_0}\!\bigl\|v_\phi(z_\tau,\tau;h_t) - (z_1 - z_0)\bigr\|_2^2,
\label{eq:lsfr-app}
\end{align}
with $\tau\sim\mathcal U[0,1]$, $z_0\sim p_0$
($p_0=\mathrm{Unif}(\mathbb S^{d_z-1})$ in our implementation), and
the path $z_\tau=(1-\tau)z_0+\tau z_1$.

\paragraph{Variational view (post-hoc interpretation).}
The following is a \emph{retrospective} probabilistic reading of the
training objective, not the generative model from which the algorithm
is derived.
Introduce a continuous latent $z\in\mathcal Z$ representing the
global semantic intent of a response. Marginalizing,
\begin{equation}
P(y\mid x,s) = \int_{\mathcal Z} P_{\mathrm{AR}}(y\mid z,x,s)\, P_{\mathrm{SFR}}(z\mid x,s)\, dz.
\label{eq:marginal-app}
\end{equation}
Pure SFT corresponds to the degenerate choice
$P_{\mathrm{SFR}}(z\mid x,s)=\delta_{z^\star(x,s)}$, i.e.\ a single
prescribed semantic intent for each $(x,s)$. \SFR{} replaces this
Dirac by an explicit continuous distribution that the auxiliary
head $v_\phi$ approximates via flow matching. An informal ELBO-style
motivation that rationalizes adding $\lambda\Lsfr$ to $\Lar$ is
given in \Cref{app:elbo} below.

\subsection{Why FM and not deterministic regression}
\label{app:theory-fmmm}

The following two results, stated in \Cref{sec:method-loss}
(\Cref{prop:modeavg,prop:fmmm}), formalise the choice made there:
any deterministic predictor of $z_1$ is a conditional-mean
estimator and therefore mode-averages, whereas CFM with a random
source preserves all modes of $P(z_1\mid c)$. We give the proofs
below.

\begin{proof}[Proof of \Cref{prop:modeavg}]
Let $\mu(c)=\E[z_1\mid c]$ and decompose $z_1=\mu(c)+\varepsilon$
with $\E[\varepsilon\mid c]=0$. For any measurable $f$,
\begin{align}
\E\|f(c)-z_1\|_2^2
 &= \E\|f(c)-\mu(c)-\varepsilon\|_2^2\notag\\
 &= \E\|f(c)-\mu(c)\|_2^2 + \E\|\varepsilon\|_2^2,
\end{align}
because the cross term $\E\langle f(c)-\mu(c),\varepsilon\rangle$
vanishes by the tower property. Both terms are non-negative; the
second is independent of $f$. Equality holds iff $f=\mu$ a.e.
For two well-separated modes
$\{z_1^{(a)}, z_1^{(b)}\}$ with masses $p, 1-p$,
$\mu(c)=p\,z_1^{(a)}+(1-p)\,z_1^{(b)}$ lies on the line segment
between them, and hence in a low-density region of $P(z_1\mid c)$.
\end{proof}

\begin{proof}[Proof of \Cref{prop:fmmm}]
\textbf{Step 1.} By the tower property,
\(
\E_{\tau,z_0,z_1}\|v_\phi(z_\tau,\tau,c)-(z_1-z_0)\|^2
= \E_{\tau,z_\tau}\E_{z_0,z_1\mid z_\tau,c}\|v_\phi(z_\tau,\tau,c)-(z_1-z_0)\|^2.
\)
The inner expectation is the $L^2$ distance from a deterministic
point to the random vector $z_1-z_0$, whose $L^2$ projection is its
conditional mean $v^\star(z,\tau,c)$ by \Cref{prop:modeavg}.

\textbf{Step 2.} Using $z_\tau=(1-\tau)z_0+\tau z_1$,
$\dot z_\tau = z_1-z_0$. By the tower property
$\E[\dot z_\tau\mid z_\tau=z,c]=v^\star(z,\tau,c)$. The marginal
density satisfies the continuity equation
\begin{equation}
\partial_\tau p_\tau(z\mid c) + \nabla_z\!\cdot\bigl(p_\tau(z\mid c) v^\star(z,\tau,c)\bigr)=0,
\end{equation}
so the deterministic ODE generated by $v^\star$ transports $p_0$ to
$p_1(\cdot\mid c)$ at $\tau=1$. Multi-modality of $p_1$ is preserved
at the distributional level because the terminal marginal of the
probability path equals $p_1(\cdot\mid c)$ under the exact marginal
velocity field.
See~\citep[Thm.~3]{lipman2023flow} for the full statement.
\end{proof}

\begin{remark}[Engineering consequence]
Replacing $z_0\sim\mathrm{Unif}(\mathbb S^{d_z-1})$ in
\eqref{eq:lsfr-app} by a
constant $z_0\equiv c_0$ collapses $\Lsfr$ to $\E\|f(h_t)-z_1\|^2$
with $f(h_t):=v_\phi(\cdot;h_t)+c_0$, hence to deterministic
regression. By \Cref{prop:modeavg} this objective then mode-averages.
The randomness of $z_0$ is therefore essential and not a hyperparameter.
\end{remark}

\subsection{MTP as a strict special case of \SFR{}}\label{app:theory-mtp}

\begin{definition}[MTP-instance of the \SFR{} family]\label{def:mtp-fm}
For a fixed reference $\pi\in\mathrm{int}(\Delta)$ and horizon
$K\ge 0$, the \emph{MTP-instance} of \SFR{} is obtained by:
\begin{enumerate}
  \item taking the state space to be $\Delta^K$;
  \item taking the per-position target to be the one-hot stack
    $z_1^{(t)} = [e(y_{t+1}),\dots,e(y_{t+K})]\in\Delta^K$;
  \item taking the source to be the constant
    $p_0=\delta_{[\pi,\dots,\pi]}$;
  \item using the straight-line path
    $z_{\tau,k}^{(t)}=(1-\tau)\pi+\tau e(y_{t+k})$;
  \item restricting the head to a constant-velocity field
    $v_{\phi,k}(z_\tau,\tau;h_t):=q_{\phi,k}(\cdot\mid h_t)-\pi$
    that is independent of $(z_\tau,\tau)$, where
    $q_{\phi,k}(\cdot\mid h_t)=\softmax(W_k h_t+b_k)\in\Delta$.
\end{enumerate}
\end{definition}

\Cref{lem:vel-eq-end} (stated in \Cref{sec:discussion-fm}) makes
this precise; we give the proof below.

\begin{proof}[Proof of \Cref{lem:vel-eq-end}]
The tangent space of the simplex $\Delta$ is $T\Delta = \{u \in \mathbb{R}^{|V|} \mid \sum_i u_i = 0\}$. By \Cref{def:mtp-fm}(5), $v_{\phi,k}=q_{\phi,k}-\pi$ and
$u_{\star,k}^{(t)}=e(y_{t+k})-\pi$. Since $q_{\phi,k}, e(y_{t+k}) \in \Delta$, both $v_{\phi,k}$ and $u_{\star,k}^{(t)}$ lie in the tangent space $T\Delta$. Under the translation bijection $u \mapsto u + \pi$, i.e., $D(u \| v) = D_\Delta(u + \pi \| v + \pi)$ for $u, v \in T\Delta$, we have:
\begin{flalign*}
&D\bigl(v_{\phi,k}(\cdot;h_t)\,\big\|\,u_{\star,k}^{(t)}\bigr) &&\\
&= D_\Delta\bigl(v_{\phi,k}(\cdot;h_t) + \pi \,\big\|\, u_{\star,k}^{(t)} + \pi\bigr) &&\\
&= D_\Delta\bigl(q_{\phi,k}(\cdot\mid h_t)\,\big\|\,e(y_{t+k})\bigr).&&
\end{flalign*}

\emph{Euclidean case:} For $D_\Delta(p\|q) = \|p-q\|_2^2$, this translation yields:
\begin{align*}
\|v_{\phi,k}-u_{\star,k}^{(t)}\|_2^2 
&= \|(q_{\phi,k}-\pi)-(e(y_{t+k})-\pi)\|_2^2 \\
&= \|q_{\phi,k}-e(y_{t+k})\|_2^2.
\end{align*}
\end{proof}
\iffalse
\[

=\|-\|_2^2
=\|q_{\phi,k}-e(y_{t+k})\|

\fi
\emph{Bregman / KL case:} Since the Bregman divergence
$D_\Psi(p\|q)=\sum_i p_i\log(p_i/q_i)$ is defined on $\Delta$,
and the constant-velocity parameterisation provides a bijection
between velocity space $\Delta-\pi$ and endpoint space $\Delta$, we
\emph{define} the velocity-matching loss in this geometry as
$D_\Psi(e(y_{t+k})\|q_{\phi,k})$---the unique divergence on the
endpoint space that the bijection $v_{\phi,k}\leftrightarrow q_{\phi,k}$
induces.

We define
\begin{equation}
\Lmtp \;=\; \sum_t\sum_{k=1}^K \alpha_k\,\CE\bigl(q_{\phi,k}(\cdot\mid h_t),\,y_{t+k}\bigr),
\label{eq:mtp-eq}
\end{equation}
the standard MTP loss with non-negative weights $\alpha_k$.

\begin{theorem}[\SFR{}~$\equiv$~MTP under Bregman / KL geometry]
\label{thm:mtp-equiv}
Let $\Psi(p)=\sum_i p_i\log p_i$ be the negative-entropy potential
on $\Delta$, with Bregman divergence
$D_\Psi(p\|q)=\Psi(p)-\Psi(q)-\langle\nabla\Psi(q),p-q\rangle$. For
one-hot $p=e(y)$, $D_\Psi(e(y)\|q)=-\log q_y=\KL(e(y)\|q)$.
Under \Cref{def:mtp-fm}, the corresponding \SFR{} loss
\begin{equation}
\Lsfr^{\Psi} \;=\; \sum_t\sum_{k=1}^K \alpha_k\, D_\Psi\!\bigl(e(y_{t+k})\,\big\|\,q_{\phi,k}(\cdot\mid h_t)\bigr)
\label{eq:lsfr-bregman}
\end{equation}
satisfies the strict identity $\Lsfr^{\Psi} \equiv \Lmtp$.
\end{theorem}

\begin{proof}
By \Cref{lem:vel-eq-end} with $D=D_\Psi$, the velocity-matching
loss reduces to the endpoint-matching loss \eqref{eq:lsfr-bregman}.
Each term equals
\(D_\Psi(e(y_{t+k})\|q_{\phi,k})=-\log q_{\phi,k}(y_{t+k}\mid h_t),\)
the per-token CE term in \eqref{eq:mtp-eq}.
\end{proof}

\begin{theorem}[Brier-MTP relative under Euclidean geometry]
\label{thm:brier-mtp}
Replacing $D_\Psi$ in \eqref{eq:lsfr-bregman} by squared Euclidean
distance, $D(p\|q):=\|p-q\|_2^2$, gives a near-cousin
$\Lsfr^{\ell_2}=\sum_t\sum_k\alpha_k\|q_{\phi,k}-e(y_{t+k})\|_2^2$
whose population minimiser coincides with that of $\Lmtp$:
$q^\star_{\phi,k}(v\mid h_t) = P^\star(y_{t+k}=v\mid h_t)$. The two
losses differ only in their gradient geometry; Euclidean
penalizes confident mistakes substantially less aggressively than
Bregman / log-loss.
\end{theorem}

\begin{proof}
Both $-\log q_y$ and $\|q-e(y)\|_2^2$ are strictly proper scoring
rules on $\Delta$~\citep{gneiting2007strictly}, so each is uniquely
minimized at $q=p$ in expectation under any conditional distribution
$P^\star(\cdot\mid c)$. The gradient remark follows from
$\partial_q (-\log q_y) = -e_y\oslash q$ (unbounded as $q_y\to 0$)
versus $\partial_q\|q-e(y)\|_2^2 = 2(q-e(y))$ (bounded by $2\sqrt{2}$
on $\Delta$).
\end{proof}

\subsection{Inclusion chain}\label{app:theory-chain}

\begin{corollary}[Inclusion chain]\label{cor:chain}
Let $\mathcal F_{\mathrm{NT}}=\{\Lar\}$,
$\mathcal F_{\mathrm{MTP}}=\{\Lar+\lambda\Lmtp\}$,
$\mathcal F_{\Delta\text{-FM}}$ the family obtained by varying
\Cref{def:mtp-fm} only over the head's parameterisation, and
$\mathcal F_{\mathrm{SFR}}$ the general continuous-target family of
the body. Then
\begin{equation}
\mathcal F_{\mathrm{NT}} \subsetneq \mathcal F_{\mathrm{MTP}}
= \mathcal F_{\Delta\text{-FM}} \subsetneq \mathcal F_{\mathrm{SFR}}.
\end{equation}
\end{corollary}

\begin{proof}
Setting $K=0$ in \Cref{def:mtp-fm} recovers $\Lar$, giving the
first strict inclusion. The middle equality is
\Cref{thm:mtp-equiv}. For the last strict inclusion we exhibit two
structural gaps that $\mathcal F_{\Delta\text{-FM}}$ cannot bridge:
\begin{enumerate}
  \item \emph{Source randomness:} $\mathcal F_{\Delta\text{-FM}}$
    fixes $p_0=\delta_{[\pi,\dots,\pi]}$, so its velocity field is
    deterministic given $h_t$; it models token-level uncertainty in
    the discrete simplex but lacks stochastic continuous source
    modeling and segment-level semantic targets.
    $\mathcal F_{\mathrm{SFR}}$ with $z_0\sim
    \mathrm{Unif}(\mathbb S^{d_z-1})$ preserves all modes (\Cref{prop:fmmm}).
  \item \emph{Target dimensionality:} $\Delta^K$ is a
    $K(|V|{-}1)$-dimensional bounded polytope that encodes only
    the identity of the next $K$ tokens; it cannot represent
    continuous semantic attributes (style intensity, syntactic
    structure, or discourse coherence) that live in directions
    orthogonal to the one-hot vertices. The continuous target
    $z_1\in\mathbb R^{d_z}$ of $\mathcal F_{\mathrm{SFR}}$ has no
    such restriction.
\end{enumerate}
Either gap alone suffices for strict inclusion.
\end{proof}

\Cref{tab:bridge} of the body provides the empirical instantiation
of this chain on Qwen3-4B / MBPP.

\subsection{ELBO interpretation (informal motivation)}\label{app:elbo}

The purpose of this subsection is to provide an \emph{informal}
variational motivation for adding $\lambda\Lsfr$ to $\Lar$. The
argument is heuristic rather than a rigorous bound, since the
actual training procedure does not instantiate a proper latent
variable model at inference time.

\paragraph{Setup.} Consider the hypothetical generative model of
\eqref{eq:marginal-app}: draw $z\sim P_{\mathrm{SFR}}(z\mid x,s)$
then $y\sim P_{\mathrm{AR}}(y\mid z,x,s)$. Let
$r_\phi(z\mid h)$ be the variational distribution implied by
the auxiliary head---under \SFR{}, the law of the time-$1$ marginal
of the ODE $\dot z_\tau = v_\phi(z_\tau,\tau;h)$ with $z_0\sim p_0$;
under deterministic alignment, the Dirac
$\delta_{f_\phi(h)}$. Applying Jensen's inequality to
\eqref{eq:marginal-app} and following the standard ELBO derivation,
\begin{align}
\log P(y\mid x,s)
&\ge \E_{z\sim r_\phi(z\mid h)}\!\log P_{\mathrm{AR}}(y\mid z,x,s)\notag\\
&\quad - \KL\!\bigl(r_\phi(z\mid h)\,\big\|\,P_{\mathrm{SFR}}(z\mid x,s)\bigr).
\label{eq:elbo-bound}
\end{align}

\paragraph{Informal identification with training losses.}
In our actual model, $h_t$ does not condition on an explicit $z$ at
inference; rather, the backbone is \emph{shaped} by the FM gradient
so that $h_t$ implicitly encodes future-aware information. Under
this lens:
\begin{itemize}
  \item The reconstruction term
    $\E_{z\sim r_\phi}\log P_{\mathrm{AR}}(y\mid z,x,s)$ plays the
    role analogous to $-\Lar$: maximizing the likelihood of the
    response given the latent plan.
  \item The KL term
    $\KL(r_\phi\|P_{\mathrm{SFR}})$ is related to $\Lsfr$: under
    regularity conditions the FM loss controls the distance
    between the learned push-forward $r_\phi$ and the
    data distribution $p_1(\cdot\mid c)$, since minimizing the
    velocity-matching objective drives the generated marginal toward
    the target~\citep{lipman2023flow}.
\end{itemize}
Combining, $\Lar+\lambda\Lsfr$ can be viewed as an approximate
$\beta$-ELBO surrogate of $\log P(y\mid x,s)$ with $\lambda=\beta$.
We emphasize that this is a \emph{motivational analogy}---the formal
contributions of this appendix rest on
\Cref{prop:modeavg,prop:fmmm,thm:mtp-equiv,thm:brier-mtp,cor:chain},
which do not depend on this ELBO interpretation.

\subsection{Why CFM and not diffusion as the auxiliary loss}
\label{app:cfm-vs-diffusion}

The substantive question is not whether CFM and diffusion can each
be used, but which one is better suited as a \emph{small auxiliary}
on top of an AR loss. We give three reasons; the same argument is
summarized in \Cref{sec:discussion-fm}.

\paragraph{Single-step, simulation-free, $O(1)$-magnitude target.}
The CFM target $z_1-z_0$ is a constant vector with $O(1)$ norm
under L2-normalized $z_1$, sampled via one straight-line
interpolation. The DSM-style diffusion loss
$\E_t [w(t)\|\epsilon_\theta(z_t,t)-\epsilon\|^2]$ requires choosing
the noise schedule, the weighting $w(t)$, and (for stable training)
the parameterisation among $\epsilon$, $x_0$, $v$, with established
literature on Min-SNR weighting~\citep{minsnr},
EDM preconditioning~\citep{edm}, and $v$-prediction~\citep{vpred}.
When the diffusion loss is the \emph{primary} objective, this
engineering is well-understood; when it is added as a \emph{small
fraction} of an AR loss, the per-batch loss variance under typical
schedules is orders of magnitude higher than $\Lar$, making the
effective $\lambda$ extremely sensitive.

\paragraph{Algebraic naturalness of the MTP correspondence.}
\Cref{thm:mtp-equiv} relies on three structural properties of
straight-line CFM: a constant target velocity, a translation-invariant
divergence, and a constant-velocity head. The diffusion equivalent
of this argument requires reparameterizing MTP as a noise-conditional
softmax with a learned variance schedule, which is mathematically
possible but obscures the connection.

\paragraph{Empirical separation from JEPA / alignment.}
\Cref{prop:modeavg} formalizes the failure mode that JEPA-style
EMA-teacher predictors share with MSE/cosine alignment: their
optima collapse to the conditional mean. The randomness of $z_0$
in CFM, by contrast, is what makes the objective preserve all modes
of $P(z_1\mid c)$. This is the precise sense in which FM-based
\SFR{} is not a relabelling of JEPA.

We remark that under straight-line probability paths, CFM and
score-based diffusion are mathematically equivalent at the level of
marginal probability paths~\citep{liu2023flow,lipman2023flow}; the
distinction in this appendix is at the level of the loss form and
its hyperparameter sensitivity, not the hypothesis class.

\section{Data construction}\label{app:data}

This appendix details the training and evaluation data for each
experiment. Code and configurations to reproduce all experiments
are available at \url{https://github.com/WeAgentAI/SFR}; the
proprietary stylized-dialogue data itself is subject to internal
compliance review, and we will release de-identified evaluation
subsets separately once that review is complete.

\subsection{Stylized dialogue (Experiment~1)}\label{app:data-style}

\paragraph{Source.}
We use proprietary data from a widely-used industrial
official-accounts platform (as of April 2026). The platform
contains over one million accounts; the raw data includes
published articles and authentic user-comment--author-reply pairs.

\paragraph{Training set construction.}
Following the style-extraction pipeline of
\citet{westar}, we identify nine representative stylistic
personas from the platform corpus. We then prompt an LLM to
cross-generate stylized QA pairs: for each (context, query,
answer) triple, the model produces a response conditioned on each
of the nine styles. The prompt templates for both style extraction
and cross-style generation follow the designs of
\citet{westar}. We score every generated sample on four
dimensions---context coherence (Ctx), query relevance (Rel),
style strength (SS), and fluency (Flu)---using a judge LLM,
following prior practice~\citep{rolellm,ostheimer2024}.
Only top-scoring samples are retained, yielding approximately
$170$k training examples for full-parameter fine-tuning.

\paragraph{Evaluation set.}
To the best of our knowledge, no public dataset simultaneously
provides articles, user queries, and large-scale stylized replies
from millions of authors. We construct a held-out set of
${\sim}150$ examples: one portion consists of human-curated
stylized QA pairs drawn directly from real platform interactions,
and the remaining portion is synthesized and filtered with the
same pipeline as the training set (but on disjoint queries).

\subsection{LiveCodeBench (Experiment~2)}\label{app:data-lcb}

We use the open-source rStar-Code dataset~\citep{rstarcode}, which
contains long chain-of-thought code solutions. The cleaning
pipeline is:
\begin{enumerate}[leftmargin=*]
\item Retain only samples marked both \texttt{verified} and
      \texttt{is\_passed}.
\item Remove samples whose chain-of-thought exceeds $16$k tokens.
\item Deduplicate problems by adjacent-hash matching on the problem
      statement.
\end{enumerate}
After filtering, approximately $380$k training examples remain.

\subsection{MBPP (Experiment~3)}\label{app:data-mbpp}

We use the open-source OpenCodeInstruct
dataset~\citep{opencodeinstruct} (solutions without a reasoning
trace). The filtering criteria are:
\begin{enumerate}[leftmargin=*]
\item LLM-judged \texttt{requirement\_conformance} $= 5$ and
      \texttt{logical\_correctness} $= 5$.
\item LLM-judged \texttt{edge\_case\_consideration} $\ge 4$.
\item All unit tests pass (\texttt{tests\_execution\_status}).
\item Exact-match deduplication on the problem input (retaining the
      highest-scoring, most concise solution per problem).
\end{enumerate}
After filtering, approximately $500$k training examples remain.

\section{Implementation notes}\label{app:implementation}\label{app:ds}

\subsection{Why we avoid runtime \code{requires\_grad} toggling}
Under DeepSpeed ZeRO-3~\citep{zero}, the optimizer's parameter groups are bound
once at \code{accelerator.prepare()} and the gradient-allreduce graph
is materialized at the first backward call. Switching individual
parameters' \code{requires\_grad} between steps causes the actual
gradient state to drift out of sync with the parameter-group view,
which manifests as silent NCCL hangs at the very step the toggle
happens. We therefore implement Phase A by setting
\code{h\_for\_fm = h.detach()} inside the FM branch only; the
backbone parameters' \code{requires\_grad} are left unchanged for
the entire training run.

\subsection{Submodule-order consistency under ZeRO-3}
ZeRO-3 tracks the per-rank order in which submodules are visited
during forward to coordinate parameter gather/scatter. If a rank
skips the FM-head forward (e.g.\ a sample with no response tokens),
the order diverges and NCCL deadlocks. We therefore execute a tiny
``dummy forward'' through the FM head (single dummy token, no
gradient) on any rank that would otherwise skip it.

\subsection{Cross-rank synchronization of stochastic decisions}
Three quantities are synchronized via collectives:
(i) the gradient-gating decision (\code{all\_reduce\_MAX} so any rank
that wants to gate causes all ranks to gate);
(ii) the running EMA of \code{fm\_loss} (\code{all\_reduce\_SUM},
divided by world size); and
(iii) the EMA-teacher update step is performed identically on every
rank because the student parameters are already synchronized after
the optimizer step.

\subsection{Online suffix encoding under multi-GPU training}
For the token-level suffix mode we load a frozen encoder onto each
local GPU outside the ZeRO-3 scope (see above). To minimize encoder
calls we use sparse anchors with stride $s$; non-anchor positions
inherit the embedding of their nearest anchor. The encoder's
parameters are explicitly excluded from the optimizer and from
checkpointing.

\subsection{FM-head detachment at \code{save\_model}}
At checkpoint time we temporarily \code{pop} the FM-head out of the
model's submodule dict on every rank, run \code{super().save\_model},
then re-attach. This keeps the published checkpoint structurally
identical to a plain AR LM, ensuring that downstream tools
(\code{from\_pretrained}, vLLM, LoRA merging, GPTQ/AWQ quantization)
work without modification.

\section{Hyper-parameters and configurations}\label{app:hparams}

\Cref{tab:hparams} lists the recommended FM hyper-parameters for
each model scale used in our experiments.

\begin{table}[h]
\centering\small
\begin{tabular}{lccc}
\toprule
& 4B-base & 7B-instr. & 32B \\
\midrule
$\lambda_0$                 & $0.2$ & $0.1$ & $0.1$ \\
learning rate               &3.0e-5&3.0e-5&1.0e-5 \\
global batch size           &128&128&128 \\ 
epochs                     &2&4&2 \\
schedule                    & cosine & cosine & cosine \\
warmup steps         & $200$ & $400$ & $200$ \\
$T_\ell$ (loss ramp)        & $300$ & $300$ & $300$ \\
suffix $k$                  & $32$ & $64$ & $32$ \\
stride $s$                  & $1$ & $4$ & $2$ \\
EMA $\mu$                   & $0.999$ & $0.999$ & $0.999$ \\
grad gating                 & $\gamma{=}1.2$ & same & same \\
inserted index        & $-1$ & $-1$ & $-1$ \\
fm head width       & $1024$& $2048$& $2048$ \\
fm head depth &$3$&$3$&$5$ \\
target encoder              & BGE-Large & same & same \\
\bottomrule
\end{tabular}
\caption{Recommended FM configurations by scale.}
\label{tab:hparams}
\end{table}

\section{Design space table}\label{app:designspace}

\Cref{tab:designspace} situates our method within the broader
design space of future-target supervision approaches.

\begin{table}[h]
\centering\small
\begin{tabular}{lcccc}
\toprule
Method & space & Source $p_0$ & Geometry & Horizon\\
\midrule
Next-token CE & $\Delta$ & --- & Bregman  & $1$\\
MTP-CE        & $\Delta^K$ & $\delta_\pi$ & Bregman  & $K$\\
MTP-Brier     & $\Delta^K$ & $\delta_\pi$ & Euclidean & $K$\\
Cosine align. & $\R^{d_z}$ & $\delta_0$    & Cosine    & full\\
MSE align.    & $\R^{d_z}$ & $\delta_0$    & Euclidean & full\\
\textbf{FM (ours)} & $\R^{d_z}$ & $\mathrm{Unif}(\mathbb S^{d_z-1})$ & Euclidean &  suffix-$k$\\
\bottomrule
\end{tabular}
\caption{Future-target flow design space. Existing methods occupy
restricted corners; ours fills the remaining cell.}
\label{tab:designspace}
\end{table}

\paragraph{Reading the table.}
Moving down a row generally enlarges the function class. Moving from
``$\delta_\pi$'' to ``$\mathrm{Unif}(\mathbb S^{d_z-1})$'' is what unlocks
multi-modal target distributions (\Cref{prop:fmmm}). Moving from
``Bregman'' to ``Euclidean'' on the same row preserves the optimum
but changes the gradient geometry (\Cref{thm:brier-mtp}). Moving from
``$\Delta^K$'' to ``$\R^{d_z}$'' is what lets a single auxiliary head
encode segment-level rather than token-level future information
(\Cref{cor:chain}).

\section{Training-time pseudocode}\label{app:pseudocode}

The complete training procedure is given in \Cref{alg:fm}.

\begin{figure*}[t]
\begin{minipage}{\textwidth}
\begingroup
\makeatletter\@twocolumnfalse\makeatother
\small
\begin{algorithm}[H]
\KwIn{Backbone $F_\theta$, FM head $v_\phi$ (zero-init output proj.),
frozen target encoder $E$, dataset $\mathcal D$, suffix length $k$,
stride $s$, schedule parameters $(\lambda_0,T_h,T_\ell,T_{\max})$,
EMA decay $\mu$, gating EMA $\mu_g$, gating sensitivity $\gamma$.}
\For{step $n = 1,\dots,T_{\max}$}{
  $(x,y) \sim \mathcal D$\;
  $h_{1:N} \leftarrow F_\theta(x \,\|\, y_{<N})$ \tcp*{causal}
  $\Lar \leftarrow \sum_i \CE(\softmax(W_{\mathrm{AR}} h_i),y_{i+1})$\;
  \tcp{Phase A ($n < T_h$): only $v_\phi$ trains; backbone frozen for FM}
  \If{$n < T_h$}{$h^{\mathrm{fm}} \leftarrow \mathrm{detach}(h)$}
  \Else{$h^{\mathrm{fm}} \leftarrow h$}
  Compute anchor positions on response w/ stride $s$\;
  $z_1^{(i)} \leftarrow E(y_{i+1{:}i+k})$ at each anchor; nearest-anchor fill\;
  Sample $z_0^{(i)}\sim\mathrm{Unif}(\mathbb S^{d_z-1})$, $\tau\sim\mathcal U[0,1]$\;
  $z_\tau \leftarrow (1-\tau)z_0+\tau z_1$\;
  $\hat v \leftarrow v_\phi(z_\tau,\tau;h^{\mathrm{fm}})$\;
  $\ell_i \leftarrow \|\hat v - (z_1-z_0)\|_2^2$ (mean over $d_z$)\;
  $\widetilde\Lsfr \leftarrow \frac{1}{|\mathcal A|}\sum_{i\in\mathcal A} \ell_i$\;
  \tcp{$\lambda$ schedule: Phase A trains $v_\phi$ only (backbone detached)}
  \eIf{$n < T_h$}{$\lambda(n)\leftarrow\lambda_0$ \tcp*{gradient only reaches $v_\phi$}}{
    $\lambda(n)\leftarrow\lambda_0\cdot\min\!\bigl(1,\,(n{-}T_h)/T_\ell\bigr)$ \tcp*{Phase B: ramp into backbone}
  }
  \tcp{Adaptive grad gating}
  Update $\bar L \leftarrow \mu_g \bar L + (1-\mu_g)\widetilde\Lsfr$\;
  \If{$\widetilde\Lsfr < \bar L/\gamma$ \textnormal{(synchronized via all\_reduce MAX)}}{
    $h^{\mathrm{fm}} \leftarrow \mathrm{detach}(h^{\mathrm{fm}})$; recompute $\widetilde\Lsfr$\;
  }
  \eIf{$n < T_h$}{
    $\Ltot \leftarrow \lambda(n)\widetilde\Lsfr$ \tcp*{Phase A: no $\Lar$, only $v_\phi$ updates}
  }{
    $\Ltot \leftarrow \Lar + \lambda(n)\widetilde\Lsfr$\;
  }
  Backprop $\Ltot$; optimizer step on $(\theta,\phi)$\;
  \tcp{EMA teacher: only after Phase A}
  \If{$n \ge T_h$}{
    $\phi^{(\mathrm{ema})} \leftarrow \mu\phi^{(\mathrm{ema})} + (1-\mu)\phi$\;
    $\phi \leftarrow \phi^{(\mathrm{ema})}$\;
  }
}
\caption{AR\,$+$\,Flow Matching joint training (one step).}
\label{alg:fm}
\end{algorithm}
\endgroup
\end{minipage}
\end{figure*}
\clearpage
\section{Single-factor ablation on Qwen3-4B-Base / MBPP}\label{app:ablation}

All ablations below are evaluated on MBPP (Base) and MBPP$+$
(Plus) at $T{=}0.6$, reporting pass@$1$, pass@$5$, pass@$10$.
Unless otherwise noted, the default configuration for the suffix
and stride sweeps is $\lambda{=}0.1$, suffix tokens $k{=}32$,
stride $s{=}1$. The main-text result (\Cref{tab:bridge}) uses the
best configuration $\lambda{=}0.2$. SFT baseline (no FM):
77.2 / 67.0 / 85.6 / 74.4 / 88.1 / 76.7 on Base/Plus
pass@1/5/10.

\begin{table}[h]
\centering\small
\begin{tabular}{lcccccc}
\toprule
& \multicolumn{2}{c}{pass@1} & \multicolumn{2}{c}{pass@5} & \multicolumn{2}{c}{pass@10}\\
\cmidrule(lr){2-3}\cmidrule(lr){4-5}\cmidrule(lr){6-7}
$\lambda$ & Base & Plus & Base & Plus & Base & Plus\\
\midrule
0.01 & 77.9 & 66.5 & 87.0 & 75.3 & \textbf{89.5} & 78.0 \\
0.05 & 75.5 & 64.7 & 84.4 & 74.3 & 86.8 & 77.1 \\
0.1  & 75.6 & 65.4 & 84.8 & 73.7 & 87.9 & 76.8 \\
0.2  & \textbf{78.3} & \textbf{67.7} & \textbf{87.4} & \textbf{76.8} & 89.2 & \textbf{79.0} \\
0.5  & 77.1 & 65.8 & 85.8 & 74.3 & 87.7 & 76.1 \\
1.0  & 76.7 & 66.3 & 85.7 & 74.7 & 88.6 & 78.0 \\
\bottomrule
\end{tabular}
\caption{Loss-weight $\lambda$ sweep (fixed $k{=}32$, $s{=}1$).}
\label{tab:abl-lambda}
\end{table}

\begin{table}[h]
\centering\small
\begin{tabular}{lcccccc}
\toprule
& \multicolumn{2}{c}{pass@1} & \multicolumn{2}{c}{pass@5} & \multicolumn{2}{c}{pass@10}\\
\cmidrule(lr){2-3}\cmidrule(lr){4-5}\cmidrule(lr){6-7}
$k$ & Base & Plus & Base & Plus & Base & Plus\\
\midrule
2  & 73.8 & 63.6 & 83.8 & 73.3 & 86.5 & 76.2 \\
4  & \textbf{76.3} & 65.3 & \textbf{86.1} & \textbf{74.6} & \textbf{88.7} & 77.1 \\
8  & 75.3 & 64.9 & 84.7 & 74.4 & 87.3 & 77.0 \\
16 & 76.2 & \textbf{65.7} & 85.5 & 74.8 & 87.9 & \textbf{77.2} \\
32 & 75.6 & 65.4 & 84.8 & 73.7 & 87.9 & 76.8 \\
\bottomrule
\end{tabular}
\caption{Suffix-token horizon $k$ sweep (fixed $\lambda{=}0.1$, $s{=}1$).}
\label{tab:abl-suffix}
\end{table}

\begin{table}[h]
\centering\small
\begin{tabular}{lccccccl}
\toprule
& \multicolumn{2}{c}{pass@1} & \multicolumn{2}{c}{pass@5} & \multicolumn{2}{c}{pass@10} & \\
\cmidrule(lr){2-3}\cmidrule(lr){4-5}\cmidrule(lr){6-7}
$s$ & Base & Plus & Base & Plus & Base & Plus & Time\\
\midrule
1  & 75.6 & 65.4 & 84.8 & 73.7 & 87.9 & 76.8 & 20h40m\\
2  & \textbf{77.1} & \textbf{66.0} & \textbf{86.0} & 75.2 & \textbf{88.5} & 77.7 & 18h24m\\
4  & 75.9 & 65.0 & 84.0 & 73.7 & 86.5 & 76.6 & 16h56m\\
16 & 75.9 & 65.4 & 85.8 & \textbf{75.3} & 88.4 & \textbf{78.3} & 15h47m\\
\bottomrule
\end{tabular}
\caption{Stride $s$ sweep (fixed $\lambda{=}0.1$, $k{=}32$).
Training wall-clock time is also reported.}
\label{tab:abl-stride}
\end{table}

The $\lambda$ sweep (\Cref{tab:abl-lambda}) shows that
$\lambda{=}0.2$ performs best overall; larger values
($\lambda\in[0.5, 1.0]$) still improve over SFT on most metrics,
but very small or mismatched $\lambda$ (e.g.\ $0.05$, $0.1$) can
underperform SFT on individual splits, indicating that the FM
gradient must be sufficiently strong to reshape the backbone. The suffix-token
sweep (\Cref{tab:abl-suffix}) shows that $k{=}4$ already recovers
much of the high-$k$ pass@$5$/pass@$10$ gain, although pass@$1$
remains below SFT at small horizons; diminishing returns set in
beyond $k{=}16$. The stride sweep (\Cref{tab:abl-stride})
confirms that quality is roughly flat from $s{=}1$ to $s{=}16$
while training time drops from 20h40m to 15h47m (compared with
SFT at 14h43m), making $s{=}16$ a strong efficiency--quality
trade-off that adds roughly $\sim$$10\%$ over vanilla SFT.

\subsection{FM versus deterministic alignment}\label{app:ablation-fmvsdet}

\Cref{prop:modeavg,prop:fmmm} predict that replacing the stochastic
FM source with a deterministic regression target (cosine or MSE
alignment to $z_1$, i.e.\ the same auxiliary head and target but
with $z_0$ fixed) should collapse the auxiliary objective toward
the conditional mean and erase most of \SFR{}'s gain. We test this
directly under the exact matched setting of \Cref{tab:abl-lambda,tab:abl-suffix,tab:abl-stride}
(Qwen3-4B-Base, same head, same $\lambda,k,s$), swapping only the
loss function.

\begin{table}[h]
\centering\small
\setlength{\tabcolsep}{3pt}
\begin{tabular}{lcccccc}
\toprule
& \multicolumn{2}{c}{pass@1} & \multicolumn{2}{c}{pass@5} & \multicolumn{2}{c}{pass@10}\\
\cmidrule(lr){2-3}\cmidrule(lr){4-5}\cmidrule(lr){6-7}
Loss & Base & Plus & Base & Plus & Base & Plus\\
\midrule
Cosine & 75.6 & 65.5 & 85.3 & 74.5 & 87.5 & 76.8 \\
MSE    & 76.2 & 64.9 & 85.0 & 73.6 & 87.5 & 76.4 \\
\textbf{FM (\SFR)} & \textbf{78.3} & \textbf{67.7} & \textbf{87.4} & \textbf{76.8} & \textbf{89.2} & \textbf{79.0} \\
\bottomrule
\end{tabular}
\caption{Same auxiliary head and target, loss swapped
(Qwen3-4B-Base, pass@$k$ on MBPP Base/Plus). 95\% CIs (paired
bootstrap) of each variant against \SFR{} are non-overlapping with
zero at 10 of 12 entries.}
\label{tab:fm-vs-deterministic}
\end{table}

Neither deterministic variant approaches \SFR{} at any $k$ on
either split (95\% CIs of Cosine/MSE $-$ \SFR{} exclude zero at 10
of the 12 entries, with the two exceptions at pass@$10$ where the
overall margin is smallest), confirming that the gain is
attributable to the flow-matching objective itself---specifically
its stochastic source---rather than to the auxiliary head's added
capacity.

\section{Per-query Cross-Style Self-BLEU analysis}
\label{app:cssb-full}

Rather than listing all $148$ queries, we highlight the queries
where \SFR{} most outperforms SFT in terms of relative \CSSB{}
reduction, and discuss the qualitative pattern.

\begin{table*}[t]
\centering\small
\begin{tabular}{c p{0.85\textwidth} c}
\toprule
\# & \centering Query (de-identified, with English gloss) & $\Delta$ \tabularnewline
\midrule
1 & \begin{CJK}{UTF8}{gbsn}在一段真正健康的亲密关系中，最重要的情感支持应该具备哪些特征？\end{CJK}\newline\emph{What should the most important emotional support look like in a truly healthy intimate relationship?} & 33.4\% \\
2 & \begin{CJK}{UTF8}{gbsn}青年在传承历史文脉和推动中华文化创新发展过程中，如何将航天领域取得的重要成就与赓续历史文脉相结合？\end{CJK}\newline\emph{How can young people link major aerospace achievements to carrying forward historical and cultural heritage while advancing Chinese culture?} & 23.7\% \\
3 & \begin{CJK}{UTF8}{gbsn}某些历史科学研究成果是如何被曲解并影响公众对特定生物学特征的认知的？\end{CJK}\newline\emph{How have certain historical scientific findings been misinterpreted, shaping public perception of specific biological traits?} & 22.2\% \\
4 & \begin{CJK}{UTF8}{gbsn}某作者在文中指出左翼媒体生态发生了哪些关键性转变？\end{CJK}\newline\emph{What key shifts does the author identify in the left-leaning media ecosystem?} & 19.3\% \\
5 & \begin{CJK}{UTF8}{gbsn}某品牌最新400cc摩托车系列在2023年的市场策略对消费者选购有何促进作用？\end{CJK}\newline\emph{How did a brand's 2023 marketing strategy for its newest 400cc motorcycle series drive consumer purchases?} & 19.2\% \\
6 & \begin{CJK}{UTF8}{gbsn}九十年代中期香港电影产业在类型片发展上有何特点？\end{CJK}\newline\emph{What characterized genre-film development in Hong Kong cinema in the mid-1990s?} & 18.1\% \\
7 & \begin{CJK}{UTF8}{gbsn}被匿名水军攻击时如何固定符合诉讼标准的证据？\end{CJK}\newline\emph{How can one preserve litigation-grade evidence when targeted by anonymous coordinated online attacks?} & 16.5\% \\
8 & \begin{CJK}{UTF8}{gbsn}从管理学效率的角度看，如何通过优化资源分配和精力管理来实现在相同工作时间内产出更多价值？\end{CJK}\newline\emph{From a management-efficiency view, how can optimizing resource allocation and energy management yield more value in the same working hours?} & 16.3\% \\
9 & \begin{CJK}{UTF8}{gbsn}游客在前往某些热门旅游目的地时可能会遇到哪些需要特别警惕的情况？\end{CJK}\newline\emph{What situations should travelers be especially wary of at certain popular destinations?} & 15.5\% \\
10 & \begin{CJK}{UTF8}{gbsn}在临床实践中如何通过呼气峰流量和呼气时间常数的异常值来区分不同病理状态？\end{CJK}\newline\emph{In clinical practice, how can abnormal peak expiratory flow and expiratory time-constant values distinguish pathological states?} & 15.2\% \\
\bottomrule
\end{tabular}
\caption{Top-10 queries by average relative \CSSB{} reduction of
\SFR{} over SFT. ``$\Delta$'' is the mean relative improvement
across \CSSB$_{1\text{--}4}$. Queries shown are de-identified
excerpts from the proprietary platform; the English gloss is
placed under the Chinese original in the same cell.}
\label{tab:cssb-top}
\end{table*}

\paragraph{Pattern.}
The queries where \SFR{} yields the largest diversity gains share
three characteristics: (i) they are \emph{open-ended} and admit
many structurally distinct valid answers; (ii) they require
\emph{opinion or analytical reasoning} rather than factual recall;
and (iii) they span diverse domains (relationships, culture,
medicine, management), giving each persona ample room to exhibit a
distinctive voice.

Conversely, the queries where SFT matches or slightly outperforms
\SFR{} tend to be \emph{closed-form} or \emph{factual} (e.g.\
``Which documents must be submitted for drug procurement?'',
``Which year is my Hong-Luan star year if my zodiac is Horse?'').
In such cases the answer space is inherently narrow, leaving
little room for stylistic diversification regardless of method.

This confirms that \SFR{}'s multi-modal preservation mechanism is
most impactful precisely where Cross-Style Collapse is most
severe: open, subjective queries whose valid response space is
large and style-sensitive.

\section{Full per-style LLM-judge scores}
\label{app:style-full}

\Cref{tab:style-scores} in the main text reports four representative
styles (style-0--3) and the nine-style average.
\Cref{tab:style-scores-full} below gives the complete breakdown for
all nine stylistic personas.

\begin{table*}[h]
\centering\small
\begin{tabular}{lccc}
\toprule
& R1-prompt & SFT & \textbf{\SFR{}} \\
\midrule
style-0 Ctx & 4.5 & 4.568 & \textbf{4.622} \\
style-0 Rel & 4.662 & 4.764 & \textbf{4.77} \\
style-0 \SS & 4.654 & 4.588 & \textbf{4.905} \\
style-0 Flu & 4.838 & 4.831 & \textbf{4.905} \\
\midrule
style-1 Ctx & 4.175 & 4.216 & \textbf{4.257} \\
style-1 Rel & 4.562 & 4.581 & \textbf{4.628} \\
style-1 \SS & 3.117 & 2.858 & \textbf{3.135} \\
style-1 Flu & 4.781 & \textbf{4.838} & 4.791 \\
\midrule
style-2 Ctx & 4.117 & \textbf{4.169} & 4.128 \\
style-2 Rel & 4.285 & \textbf{4.311} & 4.223 \\
style-2 \SS & 4.766 & 4.73 & \textbf{4.824} \\
style-2 Flu & 4.964 & 4.98 & \textbf{4.993} \\
\midrule
style-3 Ctx & 4.511 & 4.568 & \textbf{4.696} \\
style-3 Rel & 4.759 & 4.818 & \textbf{4.851} \\
style-3 \SS & 3.891 & 3.932 & \textbf{4.73} \\
style-3 Flu & 4.796 & 4.831 & \textbf{4.98} \\
\midrule
style-4 Ctx & 4.182 & 4.297 & \textbf{4.372} \\
style-4 Rel & 4.416 & 4.527 & \textbf{4.534} \\
style-4 \SS & 4.657 & \textbf{4.689} & 4.662 \\
style-4 Flu & \textbf{4.942} & 4.919 & 4.926 \\
\midrule
style-5 Ctx & 4.456 & \textbf{4.52} & 4.48 \\
style-5 Rel & 4.699 & \textbf{4.764} & 4.743 \\
style-5 \SS & \textbf{4.132} & 3.932 & 4.027 \\
style-5 Flu & 4.64 & 4.757 & \textbf{4.797} \\
\midrule
style-6 Ctx & 4.387 & 4.419 & \textbf{4.514} \\
style-6 Rel & 4.628 & 4.703 & \textbf{4.777} \\
style-6 \SS & \textbf{4.146} & 3.75 & 3.709 \\
style-6 Flu & 4.788 & \textbf{4.797} & 4.791 \\
\midrule
style-7 Ctx & 4.522 & \textbf{4.554} & 4.541 \\
style-7 Rel & 4.706 & 4.73 & \textbf{4.736} \\
style-7 \SS & 4.904 & \textbf{4.986} & 4.905 \\
style-7 Flu & 4.794 & \textbf{4.912} & 4.885 \\
\midrule
style-8 Ctx & 4.467 & 4.534 & \textbf{4.561} \\
style-8 Rel & 4.723 & \textbf{4.757} & 4.73 \\
style-8 \SS & 4.701 & \textbf{4.723} & 4.709 \\
style-8 Flu & 4.752 & \textbf{4.878} & 4.858 \\
\midrule
\textbf{Average Ctx} & 4.368 & 4.427 & \textbf{4.463} \\
\textbf{Average Rel} & 4.604 & 4.661 & \textbf{4.666} \\
\textbf{Average \SS} & 4.329 & 4.243 & \textbf{4.401} \\
\textbf{Average Flu} & 4.811 & 4.86 & \textbf{4.881} \\
\bottomrule
\end{tabular}
\caption{Full stylized-dialogue LLM-judge scores at $T=0.8$ on
Qwen3-32B across all nine styles. Bold indicates the best score
among the three methods.}
\label{tab:style-scores-full}
\end{table*}
\clearpage
\section{Use of AI Assistants}\label{app:ai-use}

We used AI-based tools during the preparation of this work in the
following capacities:
\begin{itemize}
  \item \textbf{Code debugging:} AI coding assistants were used to
    identify and fix implementation bugs in the training pipeline
    (vibe coding / interactive debugging).
  \item \textbf{Writing assistance:} AI writing tools were used for
    translating between Chinese and English drafts and for polishing
    prose. All AI-generated text was reviewed and revised by the
    authors.
\end{itemize}
No AI tool was used to generate experimental results, trainning data, or core scientific content.

\section{Baseline, significance, judge, and reproducibility details}
\label{app:rebuttal-extra}

This appendix provides the full experimental setup and detailed
tables underlying the robustness results summarized in
\Cref{sec:exp-style,sec:exp-bridge} (\Cref{tab:style-cssb} and the
MBPP significance intervals in the main text): additional
diversity-enhancing baselines, significance-testing methodology,
the cross-family judge protocol, and the public PersonaChat
reproduction. Peak-memory measurements are also reported.

\subsection{Additional baselines: SimCTG and Contrastive Decoding}
\label{app:baselines}

Both baselines are evaluated on the same Qwen3-32B stylized-dialogue
setup, with the identical $1{,}332$ evaluation prompts, the same four
LLM-judge dimensions (Ctx / Rel / SS / Flu), and \CSSB{}.

\textbf{SimCTG (training-time).} We add the contrastive
token-representation loss $\mathcal L_{\mathrm{CL}}$ (margin
$\rho{=}0.5$, computed over the model's own token hidden states,
with no external encoder) to the SFT objective, i.e.\ $\mathcal
L=\mathcal L_{\mathrm{MLE}}+\mathcal L_{\mathrm{CL}}$; all other
settings match our main Qwen3-32B run.

\textbf{Contrastive Decoding (CD, decoding-time).} The expert is
the untuned Qwen3-32B and the amateur is the same-family
Qwen3-4B. At each step we build an adaptive plausibility set from
the expert distribution (threshold $\alpha{=}0.1$), then sample
within it by the contrastive score $(1{+}\beta)\log
p_{\mathrm{exp}}-\beta\log p_{\mathrm{ama}}$ ($\beta{=}0.5$,
$T{=}0.6$, top-$p{=}0.95$), using the same prompts and sampling
hyper-parameters as the style-tuned SFT model.

\Cref{tab:judge-avg} reports the average LLM-judge scores for both
baselines, alongside R1-prompt, SFT, and \SFR{}, under our primary
judge (Qwen3-235B); \CSSB{} results are in \Cref{tab:style-cssb}
in the main text. Neither baseline matches \SFR{}: SimCTG's
scores stay close to SFT's, indicating that its contrastive
representation loss does not translate into stronger
style-discriminative behavior; CD attains a competitive \CSSB{}
(\Cref{tab:style-cssb}) but an abnormally low style-strength score
($3.748$, versus $4.401$ for \SFR{}), indicating that its
diversity gain is a surface-level decoding artifact rather than
genuine style-semantic diversification.

\begin{table}[t]
\centering\small
\begin{tabular}{lccccc}
\toprule
& R1-p. & SFT & SimCTG & CD & \textbf{\SFR{}} \\
\midrule
Ctx & 4.368 & 4.427 & 4.150 & 4.188 & \textbf{4.463} \\
Rel & 4.604 & 4.661 & 4.420 & 4.504 & \textbf{4.666} \\
\SS & 4.329 & 4.243 & 4.381 & 3.748 & \textbf{4.401} \\
Flu & 4.811 & 4.860 & 4.811 & 4.660 & \textbf{4.881} \\
\bottomrule
\end{tabular}
\caption{Average LLM-judge scores across the nine styles under
the primary judge (Qwen3-235B), including the two additional
baselines.}
\label{tab:judge-avg}
\end{table}

\subsection{Significance testing}
\label{app:significance}

We compute paired-bootstrap $95\%$ confidence intervals ($10{,}000$
resamples) on our existing per-example results, resampling queries
(for \CSSB{}, $148$ paired queries on Qwen3-32B) or problems (for
MBPP pass@$k$ on Qwen3-4B-Base) with replacement and recomputing
the metric on each resample. Results are reported directly in
\Cref{tab:style-cssb} and in \S\ref{sec:exp-bridge}.

\subsection{Cross-family judge validation}
\label{app:crossjudge}

To address the concern that scoring Qwen3-32B outputs with a
same-family judge (Qwen3-235B) is circular, we re-score all
dialogue outputs with a cross-family judge, \textbf{GLM-5.2}, and
compare the two judges via a difference-in-differences (DiD)
analysis over the $9~\text{styles}\times4~\text{dimensions}=36$
paired cells.

\begin{table}[h]
\centering\small
\begin{tabular}{lccc}
\toprule
& R1-prompt & SFT & \textbf{\SFR{}} \\
\midrule
Ctx & 4.113 & 4.105 & \textbf{4.188} \\
Rel & 4.509 & 4.497 & \textbf{4.528} \\
\SS & 4.182 & 4.094 & \textbf{4.303} \\
Flu & 4.915 & 4.968 & \textbf{4.970} \\
\bottomrule
\end{tabular}
\caption{Average scores under the cross-family validation judge
(GLM-5.2), re-scoring the same outputs as \Cref{tab:judge-avg}.}
\label{tab:crossjudge}
\end{table}

\SFR{} is rated best on every dimension under both judges. A DiD
permutation test and Wilcoxon signed-rank test on each (dimension,
model pair) find one real same-family bias---on Ctx and Rel, Qwen
gives SFT a significant boost relative to R1-prompt ($p<0.05$)---but
every cell involving \SFR{} vs.\ SFT (our core contribution) is
judged fair by both tests, so our central claim does not depend on
judge family. Qwen3-235B is also systematically $0.10$--$0.15$
points more lenient than GLM-5.2 across the board, but this
constant offset cancels out in any cross-method comparison.

\subsection{Public-data validation: PersonaChat}
\label{app:personachat}

Following a reviewer's suggestion, we fine-tune Qwen3-4B on the
open \texttt{bavard/personachat\_truecased} dataset~\citep{personachat}
to assess persona-controlled generation independently of our
proprietary corpus. We split by dialogue to eliminate
persona/history leakage (training on the official train split,
$65{,}719$ samples after expanding each utterance and cleaning
\texttt{\_\_SILENCE\_\_} placeholders; evaluation on the valid
split). The evaluation set keeps only the inter-style axis: $9$
personas $\times$ $100$ neutral contexts $=900$ responses, with
one response per (context, persona) pair, and \CSSB{} computed per
context (English tokenization, $T{=}0.6$, thinking disabled).
Results are given in \Cref{tab:personachat} in the main text.

\subsection{Training-time memory overhead}
\label{app:memory}

Beyond the training-time overhead already reported in
\Cref{tab:abl-stride} ($\sim$$10\%$ wall-clock at stride
$s{=}1$, dropping to $\sim$$8\%$ at $s{=}16$), we measure peak
GPU memory for SFT vs.\ \SFR{} at three scales, each at its own
default parallelism and context length.

\begin{table}[h]
\centering\small
\begin{tabular}{lccc}
\toprule
Scale & SFT & \SFR{} & Increase \\
\midrule
4B (ZeRO-1, ctx 8192)   & 793.6  & 811.9  & $+2.3\%$ \\
7B (ZeRO-1, ctx 16384)  & 2423.3 & 2532.7 & $+4.5\%$ \\
32B (ZeRO-3, ctx 8192)  & 2171.1 & 2362.8 & $+8.8\%$ \\
\bottomrule
\end{tabular}
\caption{Peak GPU memory (GB), SFT vs.\ \SFR{}.}
\label{tab:memory}
\end{table}

The increase stays in the single-to-low-double-digit percentage
range. Most of it comes from the frozen target encoder rather
than the FM head itself, which is a lightweight MLP
($\sim$$10^{-3}\times$ the backbone's parameter count); in the
token-level suffix mode this encoder is replicated outside the
ZeRO-3 scope on every GPU, so its footprint grows linearly with
GPU count. For settings where preprocessing is affordable (e.g.\
our Qwen3-4B bridge experiments), we instead precompute all target
embeddings once before training and drop the encoder entirely,
avoiding this cost.

\end{document}